\documentclass{article}

\usepackage{amsmath,amsfonts,bm}

\def\eqref#1{equation~\ref{#1}}

\def\1{\bm{1}}

\def\rtau{{\textnormal{$\tau$}}}

\def\rv{{\textnormal{v}}}

\def\rvv{{\mathbf{v}}}

\def\rmX{{\mathbf{X}}}
\def\rmY{{\mathbf{Y}}}

\def\vtheta{{\bm{\theta}}}
\def\vphi{{\bm{\phi}}}

\def\vb{{\bm{b}}}

\def\vg{{\bm{g}}}
\def\vh{{\bm{h}}}

\def\vx{{\bm{x}}}
\def\vy{{\bm{y}}}

\def\vomega{{\bm{\omega}}}

\def\mA{{\bm{A}}}

\DeclareMathAlphabet{\mathsfit}{\encodingdefault}{\sfdefault}{m}{sl}
\SetMathAlphabet{\mathsfit}{bold}{\encodingdefault}{\sfdefault}{bx}{n}

\newcommand{\E}{\mathbb{E}}
\newcommand{\cL}{\mathcal{L}}

\newcommand{\T}{\mathcal{T}}
\newcommand{\R}{\mathbb{R}}
\newcommand{\N}{\mathbb{N}}

\newcommand{\Var}{\mathrm{Var}}

\newcommand{\pder}[2][]{\frac{\partial#1}{\partial#2}}
\newcommand{\norm}[1]{\left\lVert#1\right\rVert}

\usepackage{microtype}
\usepackage{graphicx}
\usepackage{caption}
\usepackage{booktabs} 

\usepackage{hyperref}
\usepackage{subcaption}
\usepackage{paralist}
\usepackage[inline]{enumitem}

\usepackage[accepted]{icml2024}

\usepackage{amsmath}
\usepackage{amssymb}
\usepackage{mathtools}
\usepackage{amsthm}

\usepackage[capitalize,noabbrev]{cleveref}

\theoremstyle{plain}
\newtheorem{theorem}{Theorem}[section]

\newtheorem{lemma}[theorem]{Lemma}

\theoremstyle{definition}
\newtheorem{definition}[theorem]{Definition}

\theoremstyle{remark}

\usepackage[textsize=tiny]{todonotes}

\icmltitlerunning{
Accelerating Legacy Numerical Solvers by Non-intrusive Gradient-based Meta-solving
}

\begin{document}

\twocolumn[
    \icmltitle{
        Accelerating Legacy Numerical Solvers by Non-intrusive Gradient-based Meta-solving
    }




    \begin{icmlauthorlist}
        \icmlauthor{Sohei Arisaka}{sch,comp}
        \icmlauthor{Qianxiao Li}{sch}
    \end{icmlauthorlist}

    \icmlaffiliation{sch}{Department of Mathematics, National University of Singapore}
    \icmlaffiliation{comp}{Kajima Corporation, Japan}

    \icmlcorrespondingauthor{Sohei Arisaka}{sohei.arisaka@u.nus.edu}

    \icmlkeywords{Meta-learning, Iterative methods, Machine learning, Numerical methods, Forward gradient}

    \vskip 0.3in
]



\printAffiliationsAndNotice{}  

\begin{abstract}
Scientific computing is an essential tool for scientific discovery and engineering design, 
and its computational cost is always a main concern in practice.
To accelerate scientific computing, 
it is a promising approach to use machine learning (especially meta-learning) techniques 
for selecting hyperparameters of traditional numerical methods.
There have been numerous proposals to this direction,
but many of them require automatic-differentiable numerical methods.    
However, in reality, many practical applications still depend on 
well-established but non-automatic-differentiable legacy codes, 
which prevents practitioners from applying the state-of-the-art research to their own problems.
To resolve this problem, we propose a non-intrusive methodology with a novel gradient estimation technique
to combine machine learning and legacy numerical codes without any modification.
We theoretically and numerically show the advantage of the proposed method over other baselines
and present applications of accelerating established non-automatic-differentiable numerical solvers implemented in PETSc,
a widely used open-source numerical software library.


\end{abstract}

\section{Introduction}
\label{sec:intro}

Scientific Machine Learning (SciML) is an emerging research area, 
where scientific data and advancing Machine Learning (ML) techniques are utilized 
for new data-driven scientific discovery \cite{Baker2019-vs,Takamoto2022-mr,Cuomo2022-me}.
One of the important topics in SciML is ML-enhanced simulation 
that leverages the strengths of both ML and traditional scientific computing,
and recently numerous methods have been proposed to this direction 
\cite{Karniadakis2021-tf,Thuerey2021-xr,Vinuesa2021-id, Willard2022-vs}.
These hybrid methods can be categorized into two groups 
according to how they are trained: separately or jointly.

In the first group, ML models are trained separately from numerical solvers and combined after training.
Typically, they are trained to predict the solution of the target problem, and the solver uses the prediction as an initial guess 
\cite{Ajuria_Illarramendi2020-tb, Ozbay2021-qt, Huang2020-oa, Vaupel2020-zu, Venkataraman2021-zv}.
However, as shown in \citet{Arisaka2023-aw}, using the prediction as an initial guess while ignoring the property of the solver 
can lead to worse performance than a traditional non-learning choice. 
In addition to learning initial guesses, ML models can be used 
to learn other solver parameters such as preconditioners.
To train them, one can use certain objectives as the loss fucntion that reflect the performance of the solver, 
such as the condition number of the preconditioned matrix \cite{Sappl2019-kq, Cali2023-qw} 
and the spectral radius of the iteration matrix \cite{Luz2020-ky}.
This approach is effective and efficient because the models can be trained without running solvers. 
However, it requires some domain knowledge to design effective objectives,
and derived methods are applicable to specific problems and solvers.

In the second group, to which this paper belongs, ML models are trained jointly with running numerical solvers.
This approach is more general and straightforward than the first group, 
because the ML model can learn from actual solver's behavior through their gradients 
that represent how ML model's output affects solver's output.
For example, in \citet{Um2020-zx}, 
authors propose a framework to train neural networks with differentiable numerical solvers 
and report a significant error reduction by correcting solutions using neural networks trained with PDE solvers.
In \citet{Hsieh2018-ey} and \citet{Chen2022-tx}, neural networks are trained to learn parameters of numerical solvers 
by minimizing the residual obtained by running the iterative solver with the learned parameters.
Furthermore, \citet{Arisaka2023-aw} proposed a general gradient-based learning framework to combine ML and numerical solvers,
which can be applied to a wide range of problems and solvers.

However, a main limitation of the second training approach is that 
it requires the gradients of the outputs of the numerical solvers with resepect 
to their parameters that ML models are learning to tune.
There are two difficulties in computing the gradients.
First, despite the recent progress in learning-enhanced simulations, 
numerical computations in many practical applications are still performed by non-automatic-differentiable programs,
including established open source libraries such as OpenFOAM \citep{Jasak2009-ao} and PETSc \citep{Balay2023-yy}, 
original legacy codes developed and maintained in each community, and commercial black-box softwares
such as ANSYS \citep{Stolarski2018-ly}.
Therefore, to compute their gradients, users need to modify existing codes to add the feature
or reimplement them in deep learning frameworks for enabling automatic differentiation.
It is an arduous task and is even impossible for black-box softwares.
Second, even for differentiable numerical methods implemented in deep learning frameworks, 
it can be difficult to compute the gradients using backpropagation. 
This is because the computation process of the numerical methods can be very long, 
and the computation graph for backpropagation can be too large to fit into memory.





To overcome this limitation, we propose a novel methodology, 
called non-intrusive gradient-based meta-solving (NI-GBMS), 
to accelerate legacy numerical solvers by combining them with ML without any modification (\cref{fig:wrapper}).
Our proposed method expands the gradient-based meta-solving (GBMS) algorithm in \citet{Arisaka2023-aw} 
to handle legacy non-automatic-differentiable codes by a novel gradient estimation technique 
using an adaptive surrogate model of the legacy codes.


Our main contributions can be summarized as follows:
\begin{compactenum}
    \item In \cref{sec:method}, we propose a non-intrusive methodology to train neural networks (meta-solvers)
    for accelerating legacy numerical solvers jointly with their non-automatic-differentiable black-box routines.
    To develop the methodology, we introduce an unbaised variance reduction technique for the forward gradient 
    using an adaptive surrogate model to construct the control variates.
    \item In \cref{sec:analysis}, we theoretically and numerically show that the training using the proposed method converges faster than 
    the training using the original forward gradient in a simple problem setting.
    \item In \cref{sec:app}, we demonstrate an application of the proposed method 
    to accelerate established non-automatic-differentiable numerical solvers inmplemented in PETSc, 
    resulting in a significant speedup.
    To the best of our knowledge, this is the first successful joint training of neural networks as meta-solvers and legacy solvers.
\end{compactenum}

\section{Related Work}
\label{sec:related}
\paragraph{ML-enhanced Simulation}
ML-enhanced simulation is identified as one of the priority research directions 
in SciML \cite{Baker2019-vs}.
Among a variety of approaches to this direction, we use a meta-learning approach to combine ML and numerical solvers.
In this approach, ML models learn how to solve numerical problems rather than the solution itself.
For instance, \citet{Guo2022-au} developed a learning-enhanced Runge-Kutta method for solving ordinary differential equations.
\citet{Chen2022-tx} used a neural network to generate smoothers of the Multi-grid Network depending on each problem instance.
In particular, this paper follows the meta-solving framework, 
which is proposed in \citet{Arisaka2023-aw} to combine ML and numerical methods, 
and expand it to handle non-automatic-differentiable black-box solver implementations.

\paragraph{Forward Gradient}
The forward gradient is a projection of the gradient on a radom perturbation vector, 
and it gives a low-cost unbiased estimator of the gradient \citep{Belouze2022-km}.
It is expected to be a possible alternative to the backpropagation algorithm,
but its high variance is considered as a main limitation of the method \citep{Baydin2022-xv}.
To reduce the variance, several approaches have been proposed. 
\citet{Ren2022-gc} showed that the activity-perturbed forward gradient has lower variance 
than the original weight-perturbed forward gradient.
\citet{Bacho2022-ru} and \citet{Fournier2023-rh} proposed variance reduction methods by trading off the unbiasedness.
In this paper, we propose a novel unbaiased variance reduction approach that is orthogonal to these works.

\paragraph{Black-box Functions and Neural Networks}
Our control variate approach is inspired by \citet{Grathwohl2017-bi}, 
where the authors construct a gradient estimator using the control variates 
for a specific class of black-box loss function written as $\cL(\vtheta) = \E_{p(b|\vtheta)}[f(b)]$.
It is not applicable to deterministic cases
because their method is based on the score-function estimator \citep{Williams1992-ll}.
On the other hand, our method is based on the forward gradient and applicable to any (mathematically) differentiable function.
For more general approach, \citet{Jacovi2019-pt} proposed the Estimate and Replace approach, 
where a neural network is trained by minimizing the difference between its output and the black-box function's output
and used to bypass the black-box function during the backward computation.
This method is simple but biased whereas ours is (asymptotically) unbiased
because we use it for improving the quality of the forward gradient of the black-box function 
rather than replacing the gradient.
We will show our advantage over this simple replacement approach in \cref{sec:analysis} and \cref{sec:app}.



\section{Method: Non-intrusive Gradient-based Meta-solving}
\label{sec:method}

\begin{figure*}[tb]
    \centering
    \includegraphics[width=\textwidth]{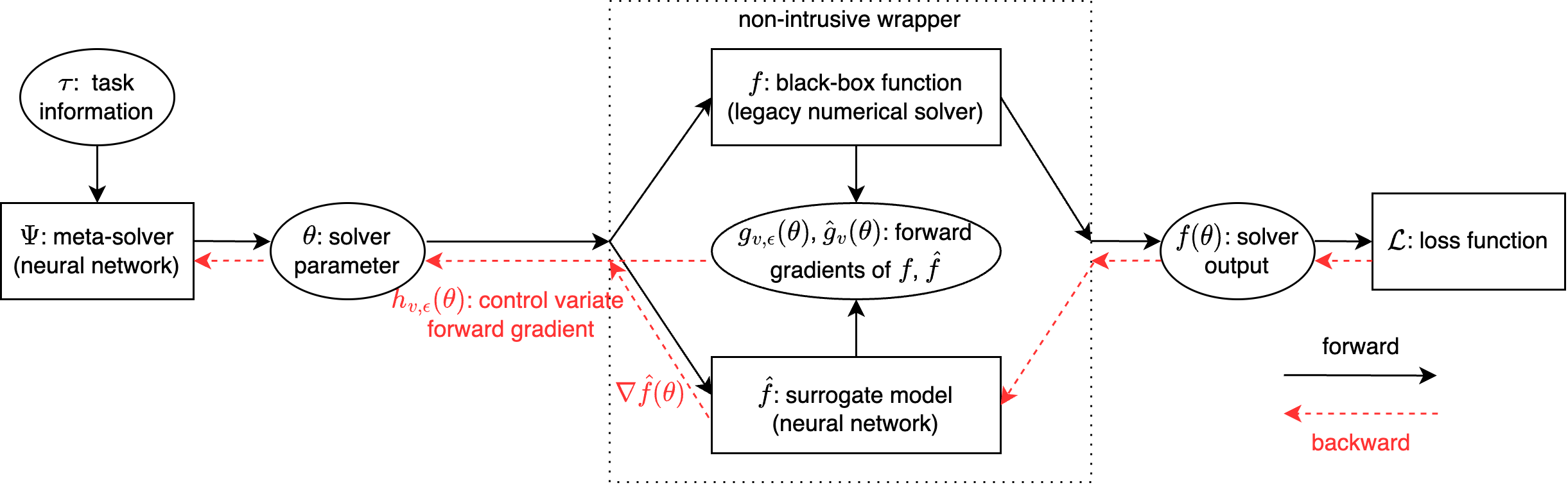}
    \caption{The overall architecture of non-intrusive gradient-based meta-solving}
    \label{fig:wrapper}
\end{figure*}

\begin{algorithm*}[tb]
    \caption{Non-intrusive gradient-based meta-solving}
    \label{alg:forward_gradient}
    \begin{algorithmic}
        \STATE {\bfseries Input:}
        $(P, \T)$: task space,
        $\Psi$: meta-solver with weight $\vomega$,
        $f$: legacy solver,
        $\hat f$: surrogate model with weight $\vphi$,
        $\mathcal L$: loss function,
        $\mathrm{Opt}$: optimizer for $\Psi$,
        $\hat{\mathrm{Opt}}$: optimizer for $\hat f$,
        $\hat P$: distribution of $\rvv$,
        $\epsilon$: positive scalar,
        $S$: stopping criterion

        \REPEAT

        \STATE $\rtau \sim P$, $\rvv \sim \hat P$
        \hfill \COMMENT {sample task $\tau$ and random vector $\rvv$ from $P$ and $\hat P$}
        
        \STATE \textbf{Forward computation:}
        \STATE \hspace{1em} $\vtheta \gets \Psi(\rtau; \vomega)$
        \hfill \COMMENT {generate solver parameter $\vtheta$ by $\Psi$}
        \STATE \hspace{1em} $y$ $\gets f(\vtheta)$, $y^+$ $\gets f(\vtheta + \epsilon \rvv)$
        \hfill \COMMENT {solve task $\tau$ twice by $f$ using parameter $\vtheta$ and $\vtheta + \epsilon \rvv$}
        \STATE \hspace{1em}  $d \gets \frac{y^+ - y}{\epsilon}$
        \hfill \COMMENT {compute the directional derivative of $f$ by finite difference approximation}
        \STATE \hspace{1em}  $\hat y, \hat d \gets \mathrm{ForwardAD}(\hat f, \vtheta, \rvv)$
        \hfill \COMMENT {compute the output and directional derivative of $\hat f$ by forward mode AD}
        \STATE \hspace{1em} $L \gets \cL (y)$, $\hat L \gets (d - \hat d)^2$
        \hfill \COMMENT {compute main loss $L$ and surrogate model loss $\hat L$}
        
        \STATE \textbf{Backward computation:}
        \STATE \hspace{1em} $\vh \gets d\rvv - \hat d \rvv + \nabla_\vtheta \hat f(\vtheta)$
        \hfill \COMMENT {compute control variate forward gradient}
        \STATE \hspace{1em} $\nabla_\vomega L \gets \mathrm{ModifiedBackprop}(L, \vomega, \vh)$
        \hfill \COMMENT {compute $\nabla_\vomega L$ by backpropagation while replacing $\nabla_\vtheta f(\vtheta)$ with $\vh$}
        \STATE \hspace{1em} $\nabla_\vphi \hat L \gets \mathrm{Backprop}(\hat L, \vphi)$
        \hfill \COMMENT {compute $\nabla_\vphi \hat L$ by backpropagation}

        \STATE \textbf{Update:}
        \STATE \hspace{1em} $\vomega \gets \mathrm{Opt}(\vomega, \nabla_\vomega L)$, $\vphi \gets \hat{\mathrm{Opt}}(\vphi, \nabla_\vphi \hat L)$
        \hfill \COMMENT {update $\vomega$ and $\vphi$ by $\mathrm{Opt}$ and $\hat{\mathrm{Opt}}$ using $\nabla_\vomega L$ and $\nabla_\vphi \hat L$, respectively}
        \UNTIL{$S$ is satisfied}
    \end{algorithmic}
\end{algorithm*}

Our non-intrusive framework can be viewed as a wrapper of a legacy numerical solver  
to be compatible with deep learning (\cref{fig:wrapper}).
To describe our method, we follow the meta-solving framework introduced in \cite{Arisaka2023-aw}.
Suppose that we have task space $(\T, P)$ and numerical solver $f$ with paramter $\vtheta$.
We want to train neural network $\Psi$ with weight $\vomega$, called meta-solver, 
to generate a good parameter $\vtheta$ of the solver $f$
for each task $\tau \in \T$.
The solver performance is measured by loss function $\cL$.
Then, the meta-solving problem is defined as follows:
\begin{definition}[Meta-solving problem \citep{Arisaka2023-aw}]
    For a given task space $(\T, P)$, loss function $\cL$, solver $f$, and meta-solver $\Psi$,
    find $\vomega$ which minimizes $\E_{\tau \sim P} \left[\cL(\tau, f(\tau; \Psi(\tau; \vomega))) \right]$.
\end{definition}
For instance, $\tau$ is to solve a linear system $\mA \vx = \vb $, $f$ is the Jacobi method starting from an initial guess $\vtheta = \vx_0$, 
$\Psi$ is a neural network to generate a good initial guess $\vtheta$ of the solver $f$ to solve $\tau$, 
and $\mathcal L$ is a surrogate of the number of iterations to converge.
To train $\Psi$, \citet{Arisaka2023-aw} proposed a gradient-based approach assuming that 
$\nabla f$ can be computed efficiently.
However, as mentioned in \cref{sec:intro}, the solver $f$ in practice can be a legacy code or commercial black-box software, 
which does not support automatic differentiation.

To compute the gradient of legacy numerical solvers without backpropagation, 
we utilize the forward gradient, a low-cost estimator of the gradient.
\begin{definition}[Forward gradient \citep{Belouze2022-km}]
    For a differentiable function $f:\R^d\to\R$ and a vector $\rvv \in \R^d$, 
    the forward gradient $\vg_\rvv:\R^d \to \R^d$ is defined as
    \begin{equation}
        \label{eq:finite difference}
        \vg_\rvv(\vtheta) = (\nabla f(\vtheta) \cdot \rvv )\rvv.
    \end{equation}
\end{definition}
In \cref{sec:method} and \cref{sec:analysis}, we describe and analyze our method using the definition for scalar functions, 
but it can be easily extended to vector-valued functions, which will be presented in \cref{sec:app}.
We also note that the original forward gradient is computed using forward mode automatic differentiation \cite{Baydin2018-gl},
but we use the following finite difference approximation for legacy numerical solvers:
\begin{equation}
    \label{eq:finite difference}
    \vg_{\rvv}(\vtheta) \approx \vg_{\rvv,\epsilon}(\vtheta) = \frac{f(\vtheta + \epsilon \rvv) - f(\vtheta)}{\epsilon} \rvv,
\end{equation}
where $\epsilon$ is a small positive scalar.

The advantage of the forward gradient is its unbiasedness and low computational cost in terms of both time and memory.
It is easily shown that 
if random vector $\rvv$'s scalar components $\rv_i$ are independent and have zero mean and unit variance for all $i$,
then the forward gradient $\vg_\rvv (\vtheta)$ is an unbiased estimator of $\nabla f(\vtheta)$,
i.e. $\E[\vg_\rvv(\vtheta)] = \nabla f(\vtheta)$.
As for the computational cost, the forward gradient of automatic-differentiable function $f$ such as neural networks 
can be computed in a single forward computation using forward mode automatic differentiation.
Even if a function $f$ is implemented in legacy codes without forward mode automatic differentiation,
the forward gradient can be computed using \cref{eq:finite difference} at the cost of evaluating $f$ twice.
Note that the computation cost is equivalent to the cost of backpropagation, but it can be computed in parallel.
Moreover, to compute the forward gradient, we do not need to store the intermediate values for backpropagation.

However, despite the above advantage, the practical use of the forward gradient is limited
because it suffers high variance, especially in high-dimensional cases.
\citet{Baydin2022-xv} shows that 
the best possible variance is achieved by choosing $\rv_i$'s to be independent Rademacher variables, 
i.e. $P(\rv_i = 1) =P(\rv_i = -1) = 1/2$ for all $i$,
in which case the variance is equal to 
\begin{equation}
    \label{eq:original variance} 
    \E[\|\vg_\rvv(\vtheta) - \nabla f(\vtheta)\|^2] = (d-1) \|\nabla f(\vtheta)\|^2,
\end{equation}
which is very large for large $d$.
This high variance in fact hinders the training process, 
and we will investigate the effect of the variance in \cref{sec:analysis} and \cref{sec:app}.

To alleviate this problem, we introduce a novel approach to reduce the variance while keeping it unbiased.
The key idea is to construct control variates using a surrogate model $\hat f$ 
whose gradient $\nabla \hat f$ is easy to compute.
\begin{definition}[Control variate forward gradient]
    For differentiable functions $f:\R^d\to\R$, $\hat f:\R^d \to \R$, and a random vector $\rvv \in \R^d$,
    the control variate forward gradient $\vh_\rvv:\R^d \to \R^d$ is defined as
    \begin{equation}
        \vh_\rvv(\vtheta) = \vg_\rvv(\vtheta) - \hat{\vg_\rvv}(\vtheta) + \E_\rvv[\hat{\vg_\rvv}(\vtheta)],
    \end{equation}
    where $\hat{\vg_\rvv}(\vtheta) = (\nabla \hat f(\vtheta) \cdot \rvv )\rvv$.
    Similarly, we define its finite difference approximation $\vh_{\rvv,\epsilon}$ using $\vg_{\rvv,\epsilon}$ as
    \begin{equation}
        \label{eq:finite difference control variate}
        \vh_{\rvv,\epsilon}(\vtheta) = \vg_{\rvv,\epsilon}(\vtheta) - \hat{\vg_\rvv}(\vtheta) + \E_\rvv[\hat{\vg_\rvv}(\vtheta)].
    \end{equation}    
\end{definition}
Note that we use the finite difference approximation only for $\vg_\rvv$ but not for $\hat{\vg_\rvv}$
because the surrogate model $\hat f$ will be implemented in deep learning frameworks with automatic differentiation.
The method of control variates is a classical and widely used variance reduction technique \cite{Nelson1990-qu}. 
In the control variates method, to reduce the variance of random variable $\rmX$ of interest,
we introduce another random variable $\rmY$ which has high correlation to $\rmX$ and known mean $\E[\rmY]$.
Then, $\rmX - \rmY + \E[\rmY]$ is used as a variance-reduced estimator of $\rmX$.
The key idea of our method is to use the forward gradient of the surrogate model $\hat \vg_\rvv$ as $\rmY$ 
so that $\E_\rvv[\hat{\vg_\rvv}(\vtheta)] = \nabla \hat f(\vtheta)$ can be computed by backpropagation.

The following theorem shows that the control variate forward gradient is still unbiased 
under the same assumption in \citet{Belouze2022-km}. 
Furthermore, if the surrogate model $\hat f$ is close to the numerical solver $f$, 
then the control variate forward gradient $\vh_\rvv$ successfully reduces the variance of the original forward gradient $\vg_\rvv$.
\begin{theorem}[Improvement by $\vh_\rvv$]
    \label{thm:property}
    If $\rv_i$'s are independent and have zero mean and unit variance,
    then $\vh_\rvv (\vtheta)$ is an unbiased estimator of $\nabla f(\vtheta)$, i.e.
    \begin{equation}
        \label{eq:expectation}
        \E[\vh_\rvv(\vtheta)] = \nabla f(\vtheta).
    \end{equation}
    Furtheremore, if $\rv_i$'s are independent Rademacher variables,
    then the mean squared deviation of $\vh_\rvv(\vtheta)$ is minimized and equal to
    \begin{equation}
        \label{eq:variance}
        \E[\|\vh_\rvv(\vtheta) - \nabla f(\vtheta)\|^2] = (d-1) \|\nabla f(\vtheta) - \nabla \hat f(\vtheta)\|^2.
    \end{equation}
\end{theorem}
The proof is found in \cref{sec:proofs}.
\cref{eq:variance} shows the improvement from \cref{eq:original variance} of the original forward gradient.
In \cref{eq:original variance}, there is no controllable term, 
whereas in \cref{eq:variance}, the variance can be reduced by choosing a good surrogate model $\hat f$.
We remark that the finite difference version $\vh_{\rvv,\epsilon}$ is not unbiased due to $\vg_{\rvv,\epsilon}$
but asymptotically unbiased as $\epsilon$ tends to $0$.
The error of the finite difference version is analyzed in \cref{app:finite difference}.

The remaining problem is how to obtain a surrogate model $\hat f$.
To this end, we propose a novel algorithm, called Non-Intrusive Gradient-Based Meta-Solving (NI-GBMS, \cref{alg:forward_gradient}).
Its main workflow is the same as the original gradient-based meta-solving \citep{Arisaka2023-aw},
but it is capable of handling legacy solver $f$ whose gradient is not accessible.
In NI-GBMS, we implement surrogate model $\hat f$ by a neural network with weight $\vphi$
and adaptively update it while training meta-solver $\Psi$.
For training $\hat f$, we use 
\begin{equation}
    \label{eq:loss}
    \hat \cL(\vphi) = \left(\nabla_\vtheta  f(\vtheta) \cdot \rvv  - \nabla_\vtheta \hat f(\vtheta; \vphi) \cdot \rvv \right)^2
\end{equation}
as a surrogate model loss.
This loss function is more suitable for our purpose 
than other losses that minimize the error of outputs $f(\vtheta)$ and $\hat f (\vtheta)$ \citep{Jacovi2019-pt}.
By contrast, our $\hat \cL$ is designed to minimize $\|\nabla f(\vtheta) - \nabla \hat f(\vtheta)\|^2 = \|\E [ {\vg_\rvv}(\vtheta)-\hat{\vg_\rvv}(\vtheta)]\|^2$, 
which determines the variance of the control variate forward gradient in \cref{eq:variance}.
Then, the control variate forward gradient is used as an approximation of $\nabla f$ 
to compute $\nabla_\vomega \cL$ for updating meta-solver $\Psi$.

We remark important characteristics of NI-GBMS.
The most significant one is its non-intrusiveness as the name suggests.
It enables us to integrate complex legacy codes and commercial black-box softwares into ML as they are.
In addition, the adaptive training strategy of the surrogate model $\hat f$ is another advantage.
By updating $\hat f$ during the training of meta-solver $\Psi$, 
it can automatically adapt to the changing distribution of $\vtheta$ generated by the meta-solver $\Psi$.
Finally, we note that the extra computational cost of training $\hat f$ is negligible 
compared to the cost of executing numerical solver $f$,
which is the bottleneck of our target problem setting.

\section{Analysis}
\label{sec:analysis}
To investigate the convergence property of NI-GBMS, we consider the simplest case, 
where a black-box solver $f$ itself is minimized over its parameter $\vtheta$, i.e. $\min_{\vtheta} f(\vtheta)$.
This is a special case of \cref{alg:forward_gradient}, 
where the meta-solver $\Psi$ is a constant function that returns its parameter $\vomega = \vtheta$,
and the loss function $\mathcal L$ is $f(\vtheta)$ itself.
This can also be considered that we directly optimize the solver parameter $\vtheta$ for a single given task $\tau$.
The optimization algorithm ${\mathrm{Opt}}$ for training $\Psi$ is a variant of the gradient descent, 
which uses the control variate forward gradient $\vh_\rvv(\vtheta)$ instead of the true gradient $\nabla f(\vtheta)$, i.e.
    $\vtheta_{k+1} = \vtheta_k - \alpha \vh_\rvv(\vtheta_k)$,
where $\alpha$ is the learning rate.

\subsection{Theoretical Analysis}
\label{sec:theory}
For the considered setting, we have the following convergence theorem.
Note that it is shown for the exact version of \cref{alg:forward_gradient}, 
where $\vh_\rvv$ is used instead of its finite difference approximation $\vh_{\rvv,\epsilon}$.
The expectation $\E[\cdot]$ in the theorem is taken over all randomness of $\rvv$ 
through the training process.

\begin{theorem}[Convergence]
    \label{thm:convergence}
    Let $f:\R^d \to \R$ be $\mu$-strongly convex and $L$-smooth.
    Denote the global minimizer of $f$ by $\vtheta_*$.
    Consider the sequence $(\vtheta_k)_{k \in \N}$ and $(\hat f_k)_{k \in \N}$ generated by NI-GBMS.
    \begin{enumerate} [label=(\alph*)]
        \item \label{thm:a}
              (\textbf{No surrogate model})
              Suppose that $\hat f_k = \hat f \equiv0$ for all $k \in \N$.
              If $\alpha \in (0, \frac{1}{Ld}]$,
              then the expected optimality gap satisfies the following inequality for all $k \in \N$:
              \begin{equation}
                    \label{eq:optimality gap}
                  \E[f(\vtheta_{k}) - f(\vtheta_*)] \leq (1 - \alpha \mu)^k (f(\vtheta_0) - f(\vtheta_*)).
              \end{equation}
        \item \label{thm:b}
              (\textbf{Fixed surrogate model})
              Instead of $\hat f\equiv0$ in \ref{thm:a}, suppose that $f_k = \hat f$ for all $k \in \N$ and $\hat f$ satisfies
              $\sup_{\vtheta\in\R^d}\{ \frac{\|\nabla f(\vtheta) - \nabla \hat f(\vtheta)\|^2}{\|\nabla f(\vtheta)\|^2}\} \leq r $ for some $r \in [0,1)$.
              Then, we have the same inequality as \cref{eq:optimality gap} 
              with an improved range of learning rate $\alpha \in (0, \min\{\frac{1}{\mu}, \frac{1}{Ldr}\}]$.
        \item \label{thm:c}
              (\textbf{Adaptive surrogate model})
              Suppose that we have
              \begin{equation}
                \label{eq:assumption}
                \E[ \|\nabla f(\vtheta_k) - \nabla \hat f_k(\vtheta_k)\|^2] \leq \xi^k \|\nabla f(\vtheta_0) - \nabla \hat f_0(\vtheta_0)\|^2
              \end{equation}
              for some $\xi \in (0,1)$.
              If $\alpha \in (0, \frac{1}{\mu}]$,
              then the expected optimality gap satisfies the following inequality for all $k \in \N$:
              \begin{equation}
                \label{eq:online}
                  \E[f(\vtheta_{k}) - f(\vtheta_*)] \leq C \rho^k,
              \end{equation}
              where
              \begin{equation}
                  C = \max\{\frac{\alpha L d \|\nabla f(\vtheta_0) - \nabla \hat f_0(\vtheta_0)\|^2}{\mu}, f(\vtheta_0) - f(\vtheta_*) \}
              \end{equation}
              and
              \begin{equation}
                  \rho = \max\{1-\alpha \mu, \xi \}.
              \end{equation}
    \end{enumerate}
\end{theorem}
The proofs are presented in \cref{sec:proofs}.    
\cref{thm:convergence} compares our method to the original forward gradient and shows our advantage.
In \ref{thm:a}, $\hat f$ does nothing and the control variate forward gradient $\vh_\rvv$ reduces to the original forward gradient $\vg_\rvv$.
This case corresponds to the original forward gradient descent \citep{Belouze2022-km}.
The best convergence rate of \ref{thm:a} is $1 - \frac{\mu}{Ld}$, 
which becomes worse as the dimension $d$ increases due to the increasing variance of $\vg_\rvv$.
The case of \ref{thm:b} represents the situation of having a fixed surrogate model $\hat f$
and shows the improvement by the control forward gradient with the surrogate model $\hat f$.
In this case, the effective best convergence rate is $1 - \frac{\mu}{Ldr}$.
Although it still depends on the dimension $d$,
it is improved by the factor of the relative error $r$ between $\nabla f$ and $\nabla \hat f$.
Note that we have $ \frac{1}{Ld} \leq \min\{\frac{1}{\mu}, \frac{1}{Ldr}\}$ because $d \geq 1$, $r \in [0,1)$, and $\mu \leq L$ \cite{Bottou2016-zs}.
The case of \ref{thm:c} represents a more ideal situation where the gradient of the surrogate model $\nabla \hat f$ is 
converging to the gradient of the target function $\nabla f$ with convergence rate $\xi$. 
In this case, $f(\vtheta_k)$ converges to $f(\vtheta_*)$ with the convergence rate $\rho =  \xi$,
which can be independent of the dimension $d$.
In summary, the convergence rate of the proposed method can be improved 
as the surrogate model $\hat f$ becomes closer to the target black-box solver $f$,
and it can be independent of the dimension $d$ in the ideal case.



\subsection{Numerical Experiments}
\label{sec:numerical}

\begin{table*}[tb]
    \caption{Objective values obtained using different gradient estimators. 
    Boldface indicates the best performance except for the case of $\nabla f$.}
    \vskip -0.05in
    \label{tab:basic}
    \begin{subtable}[t]{0.45\textwidth}
    \centering
    \caption{Sphere function}
    \label{tab:sphere}
    \begin{small}
    \begin{tabular}{l|rrrr}
        \toprule
        gradient                        & $d=2$    & $d=8$    & $d=32$   & $d=128$  \\
        \midrule
          $\nabla f$                    & 1.56e-12 & 5.88e-12 & 2.37e-11 & 9.55e-11 \\
          $\nabla \hat f$               & 9.77e-01 & 1.41e-02 & 3.67e-07 & 2.89e-02     \\
          $\vg_{\rvv,\epsilon}$         & 1.61e-06 & 6.86e-03 & 5.88e-01  & 3.68e+01 \\
          $\vh_{\rvv,\epsilon}$ (ours)  & \textbf{7.33e-12} & \textbf{8.63e-11} & \textbf{2.36e-09} & \textbf{1.36e-07} \\
          \bottomrule
       \end{tabular}
\vskip -0.2in
\end{small}
    \end{subtable}
    \hfill
    \begin{subtable}[t]{0.45\textwidth}
    \centering
    \caption{Rosenbrock function}
    \label{tab:sphere}
    \begin{small}
    \begin{tabular}{l|rrrr}
        \toprule
        gradient                        & $d=2$    & $d=8$ & $d=32$   & $d=128$ \\
        \midrule
          $\nabla f$                    & 1.18e-05 & 5.98e-01 & 5.61e-01    & 6.06e-01   \\
          $\nabla \hat f$               & 2.23e+01 & \textbf{6.64e-04}  & 7.56e-01 & 1.91e+02     \\
          $\vg_{\rvv,\epsilon}$         & 9.06e-03 & 1.08e+00  & 1.40e+01     & 1.68e+02     \\
          $\vh_{\rvv,\epsilon}$ (ours)  & \textbf{6.73e-04} & 3.40e-01 & \textbf{4.11e-01}    & \textbf{7.86e-01}   \\
          \bottomrule
        \end{tabular}
    \end{small}
\vskip -0.2in
\end{subtable}
 \end{table*}

 \begin{figure*}[tb]
    \begin{subfigure}[b]{0.45\textwidth}
        \centering
       \includegraphics[width=\textwidth]{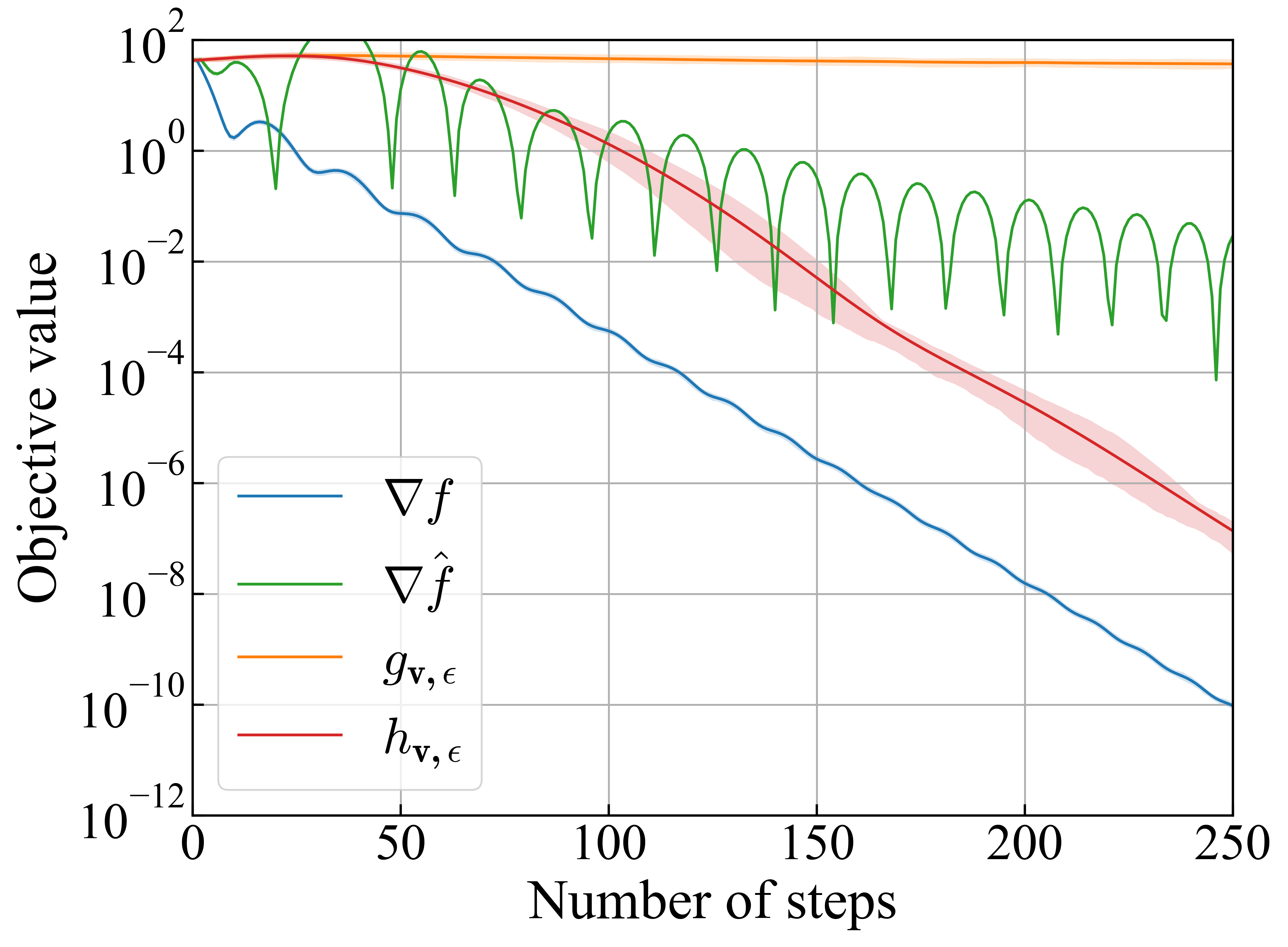}
       \caption{The convergence plot.}
       \label{fig:convergence}
    \end{subfigure}
    \hfill
    \begin{subfigure}[b]{0.45\textwidth}
        \centering
       \includegraphics[width=\textwidth]{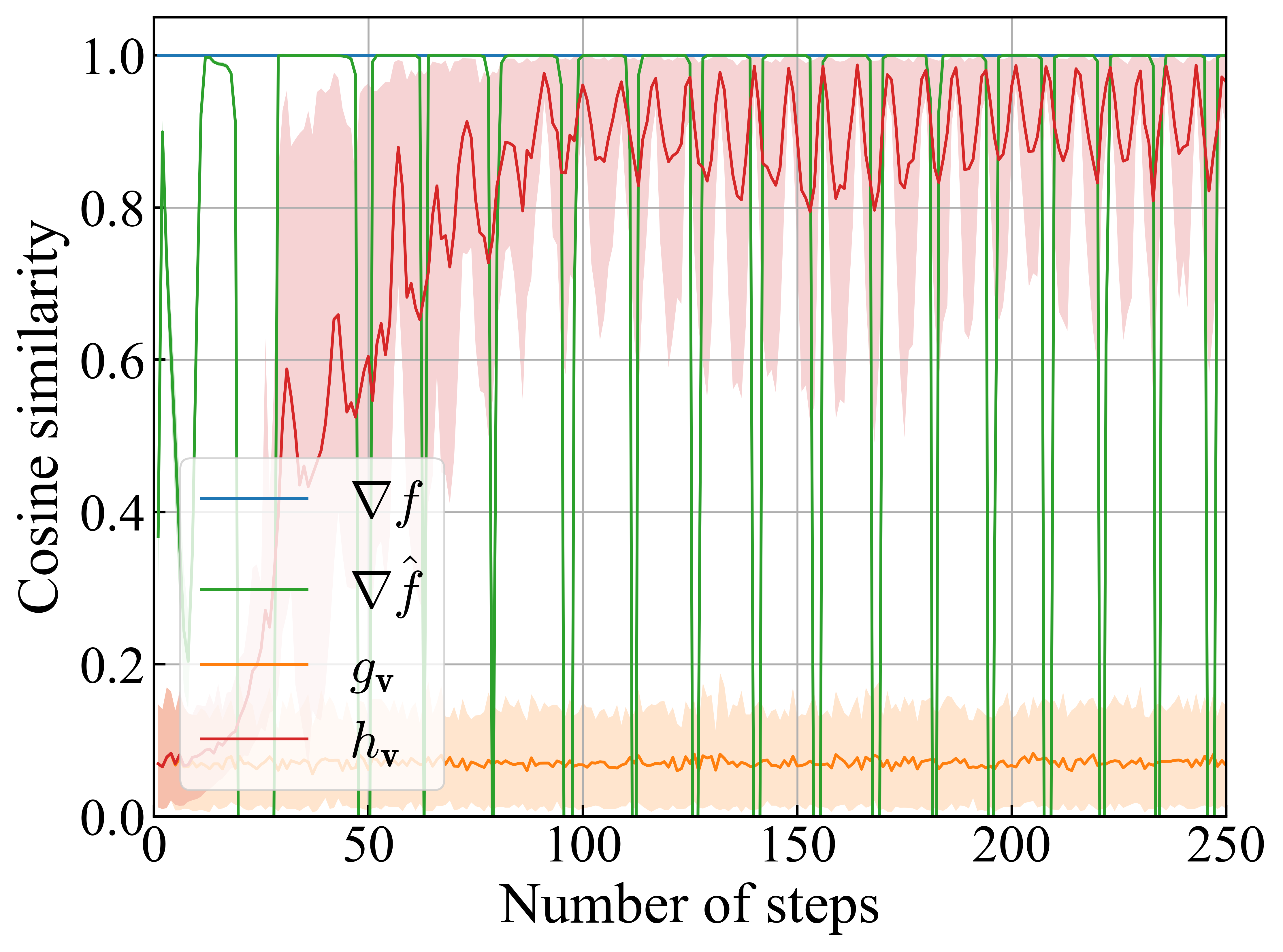}
    \caption{Cosine similarity to the true gradient $\nabla f$.}
    \label{fig:cosine}
    \end{subfigure}
    \vskip -0.1in
    \caption{The convergence plot and cosine similarity (Sphere function $d=128$).
    The shadowed area represents a range from 10\% to 90\%.}
    \label{fig:figure}
    \vskip -0.2in
\end{figure*}

Under the same problem setting, we conduct numerical experiments 
using two test functions, Sphere function \citep{Molga_undated-tx}:
\begin{equation}
    f(\vtheta)=\sum_{i=1}^{d} \theta_i^2
\end{equation}
and Rosenbrock function \citep{Rosenbrock1960-lg}:
\begin{equation}
    f(\vtheta)=\sum_{i=1}^{d-1}\left[100\left(\theta_{i+1}-\theta_{i}^2\right)^2+\left(\theta_{i}-1\right)^2\right]
\end{equation}
with the dimension $d=2, 8, 32, 128$.
For surrogate model $\hat f$, we use a convolutional neural network 
because the relation of $\vtheta$ coordinates is local.
To update the surrogate model $\hat f$, the Adam optimizer \cite{Kingma2015-ys} with learning rate 0.01 is used as $\hat {\mathrm{Opt}}$.
Note that the surrogate model $\hat f$ is trained online during minimizing $f$ without pre-training.
The main optimizer ${\mathrm{Opt}}$ is also the Adam optimizer with $\alpha=0.1$ for Sphere function 
and $\alpha=0.01$ for the Rosenbrock function.
The number of optimization steps is 250 for Sphere function and 50,000 for Rosenbrock function.
For comparison, we also test using the true gradient $\nabla f$, 
the gradient of the surrogate model $\nabla \hat f$ \citep{Jacovi2019-pt}, 
and the forward gradient $\vg_{\rvv,\epsilon}$ \citep{Belouze2022-km} instead of our control variate forward gradient $\vh_{\rvv,\epsilon}$.
We set finite difference step size $\epsilon=10^{-8}$.
The initial parameter $\vtheta_0$ is sampled from the uniform distribution on $[-1,1]^d$.
The random vector $\rvv$ is sampled from the Rademacher distribution.
We optimize 100 samples and report the average of their final objective values in \cref{tab:basic}.
Other details are presented in \cref{app:details}.



\cref{tab:basic} shows the advantage of our method.
For all cases including high dimensional cases, 
our control variate forward gradient $\vh_{\rvv, \epsilon}$ successfully minimize the objective function $f$ 
to comparable values with the best achievable baseline, the true gradient $\nabla f$.
By contrast, the performance of the original forward gradient $\vg_{\rvv, \epsilon}$ degrades quickly as the dimension increases.
The performance of the gradient of the surrogate model $\nabla \hat f$ is not stable.
Although it can perform better than the true gradient $\nabla f$ (e.g Rosenbrock function with $d=8$) thanks to its bias,
its performance is worse than our $\vh_{\rvv, \epsilon}$ for most cases.
\cref{fig:figure} illustrates the typical behavior of each gradient estimator.
In the illustrated case, the original forward gradient $\vg_{\rvv, \epsilon}$ fails to minimize the objective function $f$ because of its high variance,
which is also shown in a low cosine similarity to the true gradient $\nabla f$ in \cref{fig:cosine}.
The gradient of the surrogate model $\nabla \hat f$ shows unstable convergence behavior in \cref{fig:convergence} 
and unstable cosine similarity in \cref{fig:cosine}.
Our control variate forward gradient $\vh_{\rvv, \epsilon}$ resolves these undesired behaviors 
and shows stable convergence and high cosine similarity.
Although it takes some time for the surrogate model $\hat f$ to fit $f$, 
the convergence speed after fitting is comparable to the true gradient $\nabla f$.



\section{Application: Learning Parameters of Legacy Numerical Solvers}
\label{sec:app}



In this section, we present applications of NI-GBMS to learn good initial guesses of iterative solvers implemented in PETSc, 
the Portable, Extensible Toolkit for Scientific Computation \citep{Balay2023-yy}, 
which is widely used for large-scale scientific applications but does not support automatic differentiation.


\subsection{Toy Example: Poisson Equation}
\label{sec:poisson}

\paragraph{Problem setting}
Let us consider solving 1D Poisson equations with different source terms repeatedly.
Our task $\tau$ is to solve the 1D Poisson equation with the homogeneous Dirichlet boundary condition:
\begin{align}
    - \frac{d^2}{dz^2} u(z)  &=b_\tau(z), \quad z \in (0, 1)\\
      u(0) &= u(1)=0.    
\end{align} 
Using the finite difference scheme, 
we discretize it to obtain the linear system $\mA \vx=\vb_\tau$,
where $\mA \in \R^{N\times N}$ and $\vx,\vb_\tau \in \R^N$. 
Thus, our task $\tau$ is represented by $\tau = \{\vb_\tau\}$.
We use two task distributions $P$ and $Q$ described in \cref{app:details}.
The main loss function is a surrogate loss for the number of iterations $\tilde \cL_\delta$ presented in \citep{Arisaka2023-aw},
which is defined recursively as:
\begin{align}
    \tilde \cL_\delta^{(k+1)}(\tau; \vomega) &= \tilde \cL_\delta^{(k)}(\tau; \vomega) + \sigma(\cL_k(\tau; \vomega) - \epsilon), \\
    \tilde \cL_\delta^{(0)}(\tau; \vomega) &= 0,
\end{align}
where $\delta$ is a given target tolerance, $\sigma$ is the sigmoid function, and $\cL_k$ is the $k$-th step loss.
For $\cL_k$, we use the relative residual $\norm{\vb_\tau - \mA \vx_k}/\norm{\vb_\tau}$, where $\vx_k$ is the solution after $k$ iterations,
so the training does not require any pre-computed solutions and is conducted in a self-supervised manner.
The solver $f$ is a legacy iterative solver with an initial guess $\vtheta$ for the Poisson equation.
We use the Jacobi method $f_{\textrm{Jac}}$ and the geometric multigrid method $f_{\textrm{MG}}$ \cite{Saad2003-vm} implemented in PETSc,
and their tolerance $\delta$ is set to $10^{-3}$ and $10^{-8}$ respectively.
The meta-solver $\Psi$ parameterized by $\vomega$ generates the initial guess $\vtheta_\tau$ for each task $\tau$.
We use a fully-connected neural network $\Psi_{\textrm{GBMS}}$ for $\Psi$.
We implement surrogate model $\hat f$ by a fully-connected neural network.
Then, we train $\Psi_{\textrm{GBMS}}$ to minimize $\tilde \cL_\delta$ using \cref{alg:forward_gradient}.
Other details, including network architecture, the parameters of PETSc solvers, and training hyperparameters, are presented in \cref{app:details}.

\paragraph{Baselines}
For the purpose of comparison, we have five baselines: $\Psi_0$, $\Psi_\mathrm{SL}$, 
$\Psi_{\textrm{GBMS}}$ trained with $\nabla f$,
$\Psi_{\textrm{GBMS}}$ trained with $\nabla \hat f$,
and $\Psi_{\textrm{GBMS}}$ trained with $\vg_{\rvv,\epsilon}$.
$\Psi_0$ is a non-learning baseline that always gives initial guess $\vtheta = \bf{0}$,
which is a reasonable choice because $u_\tau(0)=u_\tau(1)=0$ and $\E_{\tau \sim P}[u_\tau] = \mathbf 0$.
$\Psi_\mathrm{SL}$ is an ordinary supervised learning baseline 
whose network architecture is the same as $\Psi_{\mathrm{NN}}$.
It is trained independently of the solver $f$ 
by minimizing the relative error $\norm{\hat \vx - \vx_{\tau*}}/\norm{\vx_{\tau*}}$, 
where $\hat \vx$ is the prediction of $\Psi$ and $\vx_{\tau*}$ is the pre-computed reference solution of task $\tau$.
Note that $\Psi_\mathrm{SL}$ does not require the gradient of the solver $f$.
$\Psi_{\textrm{GBMS}}$ trained with $\nabla f$ is considered as the best achievable baseline,
which is available when the solver $f$ is implemented in deep learning frameworks.
$\Psi_{\textrm{GBMS}}$ trained with $\nabla \hat f$ is a baseline trained using the gradient of the surrogate model $\nabla \hat f$,
which can be computed easily by backpropagation but is biased.
This corresponds to the training approach proposed in \citet{Jacovi2019-pt}.
$\Psi_{\textrm{GBMS}}$ trained with $\vg_{\rvv,\epsilon}$ is a baseline trained with the forward gradient $\vg_{\rvv,\epsilon}$,
which is applicable to legacy solvers but suffers high variance.
We also check the difference between the exact forward gradients ($\vg_\rvv, \vh_\rvv$) 
and the finite difference approximation ($\vg_{\rvv,\epsilon}, \vh_{\rvv,\epsilon}$).
For training the baselines requiring automatic differentiation, 
we implement the Jacobi method and the geometric multigrid method using PyTorch 
so that they are differentiable and equivalent to the solvers in PETSc.


\paragraph{Results}
\begin{table}[tb]
    \caption{The average number of iterations to converge. Boldface indicates the best performance except for the case of $\nabla f$. }
    \label{tab:number of iterations}
    \centering
    \begin{small}
    \begin{tabular}{ll|rr}
        \toprule
        $\Psi$               & gradient                 & $f = f_{\mathrm{Jac}}$ & $f = f_{\mathrm{MG}}$ \\
        \midrule
        $\Psi_0$             & -                                &  173.00           &  90.06    \\
        $\Psi_\mathrm{SL}$   & -                                &  175.71           &    78.62  \\
        $\Psi_{\mathrm{GBMS}}$ & $\nabla f$                     &  53.64            &  39.29        \\
                             & $\nabla \hat f$                  &  106.65           &     52.59     \\
                             & $\vg_\rvv$                       &   66.76           &     59.06    \\
                             & $\vg_{\rvv,\epsilon}$            &   66.85           &    59.09  \\
                             & $\vh_\rvv$                       &  \textbf{44.92}   &    \textbf{42.76}     \\
                             & $\vh_{\rvv,\epsilon}$ (ours)     &   45.15           &    43.06  \\
        \bottomrule
    \end{tabular}
    \end{small}
\end{table}
\begin{figure}[tb]
        \centering
       \includegraphics[width=0.45\textwidth]{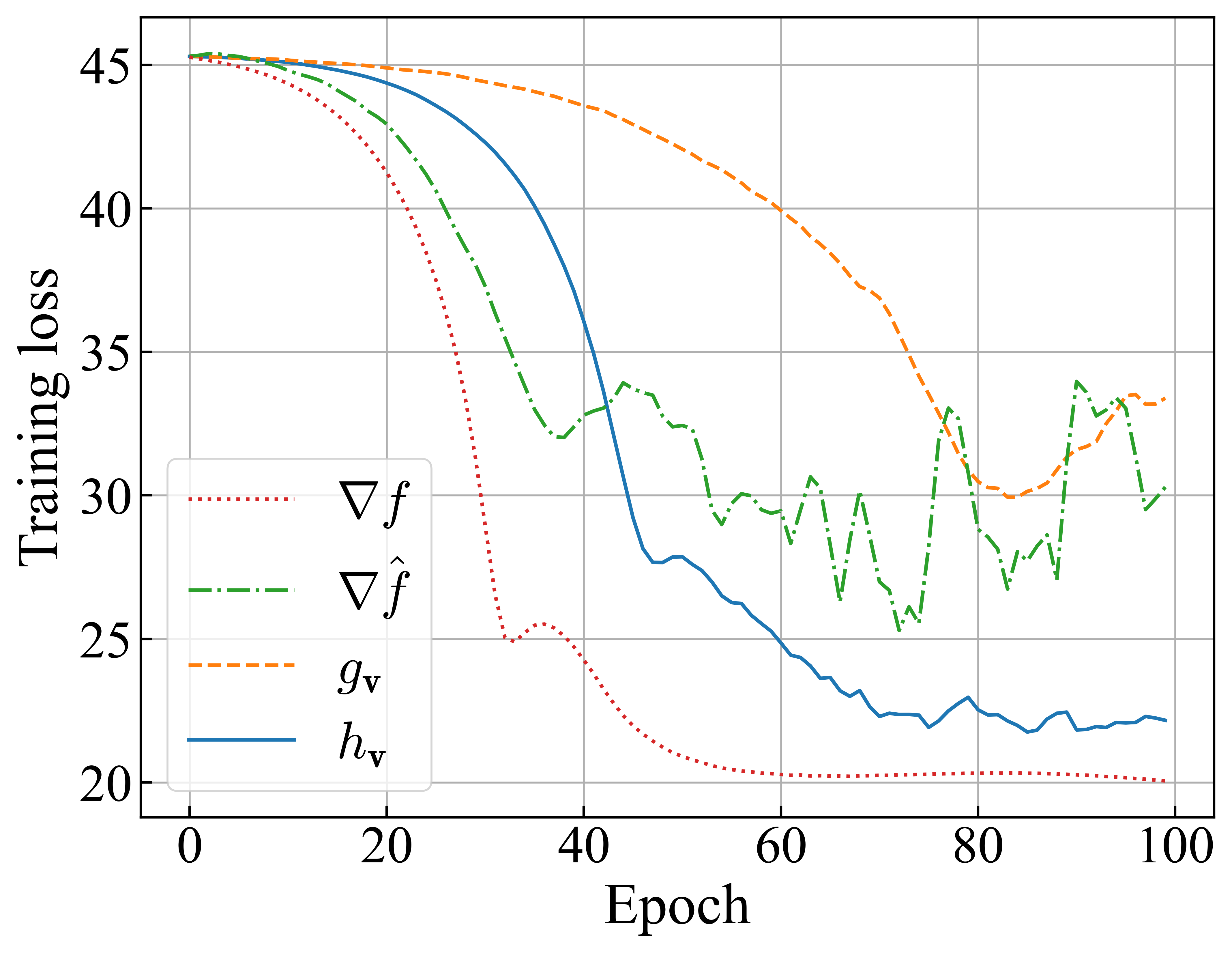}
    \caption{The learning curves of the case $f_{\mathrm{MG}}$. 
    They are trained using a same random seed and hyperparameters except for the gradient estimator.}
    \label{fig:learning curve}
\end{figure}

\cref{tab:number of iterations} presents the average number of iterations to converge 
starting from the initial guesses generated by each trained meta-solver.
For both Jacobi and multigrid solvers, the proposed method achieves comparable results to 
the best achievable baseline ($\Psi_{\textrm{GBMS}}$ trained with $\nabla f$)
and outperforms other baselines by a large margin.
The gap between $\Psi_\mathrm{SL}$ and $\Psi_{\textrm{GBMS}}$ trained with $\nabla f$
shows the importance of incorporating the solver information into the training of $\Psi$
and necessity of the proposed method for legacy solvers.
\cref{tab:number of iterations} also shows that 
there is no significant difference between the exact forward gradients and the finite difference approximation.
\cref{fig:learning curve} shows the learning curves of the case of the multigrid method.
Their behaviors are similar to the ones shown in \cref{fig:convergence}:
$\vg_{\rvv,\epsilon}$ suffers slow training, $\nabla \hat f$ is unstable, 
and $\vh_{\rvv,\epsilon}$ converges similarly to $\nabla f$ after some delay for fitting $\hat f$.
In summary, $\Psi_{\textrm{GBMS}}$ is successfully trained jointly with non-automatic-differentiable solvers using the proposed method,
and reduces the number of iterations from the solver-independent supervised learning baseline $\Psi_{\mathrm{SL}}$ by 74\% for the Jacobi method
and 45\% for the multigrid method.


\subsection{Advanced Examples}
\label{sec:advanced}
To demonstrate the versatility and performance of NI-GBMS, 
we also consider more advanced examples, 2D biharmonic equation and 3D linear elasticity equations, using FEniCS \citep{Alnaes2015-nz}.
In these examples, we accelerate the algebraic multigrid method $f_{\mathrm{AMG}}$ in PETSc, 
which is one the most practical linear solvers \citep{Stuben2001-sj}, by learning initial guesses using NI-GBMS.
Details are presented in \cref{app:details}.

\paragraph{Biharmonic Equation}
The biharmonic equation is a fourth-order elliptic equation of the form:
    $\nabla^4 u = b_\tau$ in $\Omega$,
where $\nabla^4$ is the biharmonic operator, $\Omega$ is a domain of interest.
Let $\Omega$ be $[0, 1]^2$ and boundary conditions be
$u = \nabla^2 u = 0$ on $\partial \Omega$.
The source term $b_\tau$ is sampled from a task distribution $P$ presented in \cref{app:details}.
The equation is discretized using finite elements, 
and the discretized linear system is solved by the algebraic multigrid method $f_{\mathrm{AMG}}$ in PETSc.

\paragraph{Linear Elasticity Equations}
The linear elasticity problem is described by the following equations:
\begin{align}
    -\nabla \cdot \sigma(u) & =b_\tau \quad \text { in } \Omega,                                     \\
    \sigma(u)               & =\lambda \operatorname{tr}(\epsilon(u)) I+2 \mu \epsilon(u) \\
    \epsilon(u)             & =\frac{1}{2}\left(\nabla u+(\nabla u)^T\right)
\end{align}
where $\Omega = [0, 1]^3$ is an elastic body, $u$ is the displacement field, $\sigma$ is the stress tensor, $\epsilon$ is strain-rate tensor,
$\lambda$ and $\mu$ are the Lam\'{e} elasticity parameters, and $b_\tau$ is the body force.
We consider clamped bodies deformed by their own weight by letting $u = (0, 0, 0)^T$ at $x=0$ and $b_\tau = (0, 0, -\rho_\tau g)$,
where $\rho_\tau$ is the density and $g$ is the gravitational acceleration.
The density $\rho_\tau$ is sampled from the log uniform distribution on $[10^{-2}, 10^{2}]$, which determines our task space.
Then, the equations are discretized using finite elements and solved by $f_{\mathrm{AMG}}$.

\paragraph{Results}
The performance improvement by NI-GBMS is shown in \cref{tab:advanced},
where the average number of iterations is reduced from the default initial guess, which is the zero vector, 
by 57\% for the biharmonic equation and 71\% for the linear elasticity equations.
This result demonstrates the versatility and performance of NI-GBMS for more advanced problem settings.

\begin{table}[tb]
    \caption{The average number of iterations to reach relative residual tolerance $\delta = 10^{-5}$.}
    \label{tab:advanced}
    
    \centering
    \begin{small}
    \begin{tabular}{l|rr}
        \toprule
                                 & Biharmonic Eq.               & Elasticity Eq.  \\
        \midrule
        Default PETSc            & 73.74                        &  88.71    \\
        Enhanced by NI-GBMS      & \textbf{31.54}               &   \textbf{25.61} \\
        \bottomrule
    \end{tabular}
    \end{small}
\end{table}

\section{Conclusion}
\label{sec:conclusion}
In this paper, we proposed NI-GBMS, a novel methodology to combine meta-solvers parametrized by neural networks 
and non-automatic-differentiable legacy solvers without any modification.
To develop this, we introduced the control variate forward gradient,
which is unbiased and has lower variance than the original forward gradient.
Furthermore, we proposed a practical algorithm to construct the control variate forward gradient using an adaptive surrogate model.
It was theoretically and numerically shown that the proposed method has better convergence property than other baselines.
We applied NI-GBMS to learn initial guesses of established iterative solvers in PETSc,
resulting in a significant reduction of the number of iterations to converge.
Our proposed method expands the range of applications of neural networks, 
and it can be a fundamental building block to combine modern neural networks and legacy numerical solvers.

In future work, we will explore more sophisticated architectures of the surrogate model, 
which can leverage our prior knowledge of the solver and problem and lead to better performance.
For example, one can build surrogate model $\hat f = \hat f_2 \circ \hat f_1$ using two building blocks $\hat f_1$ and $\hat f_2$, 
where $\hat f_1$ is the same algorithm as legacy solver $f$ but implemented in a deep learning framework and working on a smaller problem size,
and $\hat f_2$ is a neural network to correct the difference between $f$ and $\hat f_1$.
Beyond scientific computing applications, our proposed method can be useful to the problems of learning to optimize \cite{Chen2021-zi},
where memory limitation often becomes a practical bottleneck because the computation graph tends to be too large for backpropagation.
We expect that this limitation can be overcome by the proposed method using the high-quality gradient estimation without storing intermediate values.

\section*{Impact Statements}
This paper presents work whose goal is to advance the field of Machine Learning. 
There are many potential societal consequences of our work, 
none which we feel must be specifically highlighted here.

\section*{Acknowledgements}
We would like to thank the anonymous reviewers for their constructive comments.
S. Arisaka is supported by Kajima Corporation, Japan.
Q. Li is supported by the National Research Foundation, Singapore, under the NRF fellowship (NRF-NRFF13-2021-0005).

\newpage

\bibliography{example_paper}

\begin{thebibliography}{45}
\providecommand{\natexlab}[1]{#1}
\providecommand{\url}[1]{\texttt{#1}}
\expandafter\ifx\csname urlstyle\endcsname\relax
  \providecommand{\doi}[1]{doi: #1}\else
  \providecommand{\doi}{doi: \begingroup \urlstyle{rm}\Url}\fi

\bibitem[noa()]{noauthor_undated-zc}
{DOLFIN documentation --- DOLFIN documentation}.
\newblock
  \url{https://fenicsproject.org/olddocs/dolfin/2019.1.0/python/index.html}.
\newblock URL
  \url{https://fenicsproject.org/olddocs/dolfin/2019.1.0/python/index.html}.
\newblock Accessed: 2024-2-2.

\bibitem[Ajuria~Illarramendi et~al.(2020)Ajuria~Illarramendi, Alguacil,
  Bauerheim, Misdariis, Cuenot, and Benazera]{Ajuria_Illarramendi2020-tb}
Ajuria~Illarramendi, E., Alguacil, A., Bauerheim, M., Misdariis, A., Cuenot,
  B., and Benazera, E.
\newblock {Towards an hybrid computational strategy based on Deep Learning for
  incompressible flows}.
\newblock In \emph{{AIAA AVIATION 2020 FORUM}}, AIAA AVIATION Forum. American
  Institute of Aeronautics and Astronautics, June 2020.
\newblock \doi{10.2514/6.2020-3058}.
\newblock URL \url{https://doi.org/10.2514/6.2020-3058}.

\bibitem[Aln{\ae}s et~al.(2015)Aln{\ae}s, Blechta, Hake, Johansson, Kehlet,
  Logg, Richardson, Ring, Rognes, and Wells]{Alnaes2015-nz}
Aln{\ae}s, M., Blechta, J., Hake, J., Johansson, A., Kehlet, B., Logg, A.,
  Richardson, C., Ring, J., Rognes, M.~E., and Wells, G.~N.
\newblock {The FEniCS Project Version 1.5}.
\newblock \emph{Anschnitt}, 3\penalty0 (100), December 2015.
\newblock ISSN 0003-5238.
\newblock \doi{10.11588/ans.2015.100.20553}.
\newblock URL
  \url{https://journals.ub.uni-heidelberg.de/index.php/ans/article/view/20553}.

\bibitem[Arisaka \& Li(2023)Arisaka and Li]{Arisaka2023-aw}
Arisaka, S. and Li, Q.
\newblock {Principled Acceleration of Iterative Numerical Methods Using Machine
  Learning}.
\newblock In Krause, A., Brunskill, E., Cho, K., Engelhardt, B., Sabato, S.,
  and Scarlett, J. (eds.), \emph{{Proceedings of the 40th International
  Conference on Machine Learning}}, volume 202 of \emph{Proceedings of Machine
  Learning Research}, pp.\  1041--1059. PMLR, 2023.
\newblock URL \url{https://proceedings.mlr.press/v202/arisaka23a.html}.

\bibitem[Bacho \& Chu(2022)Bacho and Chu]{Bacho2022-ru}
Bacho, F. and Chu, D.
\newblock {Low-Variance Forward Gradients using Direct Feedback Alignment and
  Momentum}.
\newblock December 2022.
\newblock URL \url{http://arxiv.org/abs/2212.07282}.

\bibitem[Baker et~al.(2019)Baker, Alexander, Bremer, Hagberg, Kevrekidis, Najm,
  Parashar, Patra, Sethian, Wild, Willcox, and Lee]{Baker2019-vs}
Baker, N., Alexander, F., Bremer, T., Hagberg, A., Kevrekidis, Y., Najm, H.,
  Parashar, M., Patra, A., Sethian, J., Wild, S., Willcox, K., and Lee, S.
\newblock {Workshop Report on Basic Research Needs for Scientific Machine
  Learning: Core Technologies for Artificial Intelligence}.
\newblock \emph{DOE}, February 2019.
\newblock \doi{10.2172/1478744}.
\newblock URL \url{https://www.osti.gov/servlets/purl/1478744}.

\bibitem[Balay et~al.(2023)Balay, Abhyankar, Adams, Benson, Brown, Brune,
  Buschelman, Constantinescu, Dalcin, Dener, Eijkhout, Faibussowitsch, Gropp,
  Hapla, Isaac, Jolivet, Karpeev, Kaushik, Knepley, Kong, Kruger, May, McInnes,
  Mills, Mitchell, Munson, Roman, Rupp, Sanan, Sarich, Smith, Zampini, Zhang,
  Zhang, and Zhang]{Balay2023-yy}
Balay, S., Abhyankar, S., Adams, M.~F., Benson, S., Brown, J., Brune, P.,
  Buschelman, K., Constantinescu, E.~M., Dalcin, L., Dener, A., Eijkhout, V.,
  Faibussowitsch, J., Gropp, W.~D., Hapla, V., Isaac, T., Jolivet, P., Karpeev,
  D., Kaushik, D., Knepley, M.~G., Kong, F., Kruger, S., May, D.~A., McInnes,
  L.~C., Mills, R.~T., Mitchell, L., Munson, T., Roman, J.~E., Rupp, K., Sanan,
  P., Sarich, J., Smith, B.~F., Zampini, S., Zhang, H., Zhang, H., and Zhang,
  J.
\newblock {{PETS}c {W}eb page}.
\newblock \url{https://petsc.org/}, 2023.
\newblock URL \url{https://petsc.org/}.

\bibitem[Baydin et~al.(2018)Baydin, Pearlmutter, Radul, and
  Siskind]{Baydin2018-gl}
Baydin, A.~G., Pearlmutter, B.~A., Radul, A.~A., and Siskind, J.~M.
\newblock {Automatic Differentiation in Machine Learning: a Survey}.
\newblock \emph{Journal of machine learning research: JMLR}, 18\penalty0
  (153):\penalty0 1--43, 2018.
\newblock ISSN 1532-4435, 1533-7928.
\newblock URL \url{https://jmlr.org/papers/v18/17-468.html}.

\bibitem[Baydin et~al.(2022)Baydin, Pearlmutter, Syme, Wood, and
  Torr]{Baydin2022-xv}
Baydin, A.~G., Pearlmutter, B.~A., Syme, D., Wood, F., and Torr, P.
\newblock {Gradients without Backpropagation}.
\newblock February 2022.
\newblock URL \url{http://arxiv.org/abs/2202.08587}.

\bibitem[Belouze(2022)]{Belouze2022-km}
Belouze, G.
\newblock {Optimization without Backpropagation}.
\newblock September 2022.
\newblock URL \url{http://arxiv.org/abs/2209.06302}.

\bibitem[Bottou et~al.(2016)Bottou, Curtis, and Nocedal]{Bottou2016-zs}
Bottou, L., Curtis, F.~E., and Nocedal, J.
\newblock {Optimization Methods for Large-Scale Machine Learning}.
\newblock June 2016.
\newblock URL \url{http://arxiv.org/abs/1606.04838}.

\bibitem[Cal{\`\i} et~al.(2023)Cal{\`\i}, Hackett, Lin, Shanahan, and
  Xiao]{Cali2023-qw}
Cal{\`\i}, S., Hackett, D.~C., Lin, Y., Shanahan, P.~E., and Xiao, B.
\newblock {Neural-network preconditioners for solving the Dirac equation in
  lattice gauge theory}.
\newblock \emph{Physical Review D}, 107\penalty0 (3):\penalty0 034508, February
  2023.
\newblock \doi{10.1103/PhysRevD.107.034508}.
\newblock URL \url{https://link.aps.org/doi/10.1103/PhysRevD.107.034508}.

\bibitem[Chen et~al.(2021)Chen, Chen, Chen, Heaton, Liu, Wang, and
  Yin]{Chen2021-zi}
Chen, T., Chen, X., Chen, W., Heaton, H., Liu, J., Wang, Z., and Yin, W.
\newblock {Learning to optimize: A primer and A benchmark}.
\newblock March 2021.
\newblock URL \url{https://jmlr.org/papers/volume23/21-0308/21-0308.pdf}.

\bibitem[Chen et~al.(2022)Chen, Dong, and Xu]{Chen2022-tx}
Chen, Y., Dong, B., and Xu, J.
\newblock {Meta-MgNet: Meta multigrid networks for solving parameterized
  partial differential equations}.
\newblock \emph{Journal of computational physics}, 455:\penalty0 110996, April
  2022.
\newblock ISSN 0021-9991.
\newblock \doi{10.1016/j.jcp.2022.110996}.
\newblock URL
  \url{https://www.sciencedirect.com/science/article/pii/S0021999122000584}.

\bibitem[Cuomo et~al.(2022)Cuomo, Di~Cola, Giampaolo, Rozza, Raissi, and
  Piccialli]{Cuomo2022-me}
Cuomo, S., Di~Cola, V.~S., Giampaolo, F., Rozza, G., Raissi, M., and Piccialli,
  F.
\newblock {Scientific Machine Learning Through Physics--Informed Neural
  Networks: Where we are and What's Next}.
\newblock \emph{Journal of scientific computing}, 92\penalty0 (3):\penalty0 88,
  July 2022.
\newblock ISSN 0885-7474, 1573-7691.
\newblock \doi{10.1007/s10915-022-01939-z}.
\newblock URL \url{https://doi.org/10.1007/s10915-022-01939-z}.

\bibitem[Elfwing et~al.(2018)Elfwing, Uchibe, and Doya]{Elfwing2018-ae}
Elfwing, S., Uchibe, E., and Doya, K.
\newblock {Sigmoid-weighted linear units for neural network function
  approximation in reinforcement learning}.
\newblock \emph{Neural networks: the official journal of the International
  Neural Network Society}, 107:\penalty0 3--11, November 2018.
\newblock ISSN 0893-6080, 1879-2782.
\newblock \doi{10.1016/j.neunet.2017.12.012}.
\newblock URL \url{http://dx.doi.org/10.1016/j.neunet.2017.12.012}.

\bibitem[Fournier et~al.(2023)Fournier, Rivaud, Belilovsky, Eickenberg, and
  Oyallon]{Fournier2023-rh}
Fournier, L., Rivaud, S., Belilovsky, E., Eickenberg, M., and Oyallon, E.
\newblock {Can Forward Gradient Match Backpropagation?}
\newblock June 2023.
\newblock URL \url{http://arxiv.org/abs/2306.06968}.

\bibitem[Grathwohl et~al.(2017)Grathwohl, Choi, Wu, Roeder, and
  Duvenaud]{Grathwohl2017-bi}
Grathwohl, W., Choi, D., Wu, Y., Roeder, G., and Duvenaud, D.
\newblock {Backpropagation through the Void: Optimizing control variates for
  black-box gradient estimation}.
\newblock October 2017.
\newblock URL \url{https://openreview.net/pdf?id=SyzKd1bCW}.

\bibitem[Guo et~al.(2022)Guo, Dietrich, Bertalan, Doncevic, Dahmen, Kevrekidis,
  and Li]{Guo2022-au}
Guo, Y., Dietrich, F., Bertalan, T., Doncevic, D.~T., Dahmen, M., Kevrekidis,
  I.~G., and Li, Q.
\newblock {Personalized Algorithm Generation: A Case Study in Learning ODE
  Integrators}.
\newblock \emph{SIAM Journal of Scientific Computing}, 44\penalty0
  (4):\penalty0 A1911--A1933, August 2022.
\newblock ISSN 1064-8275.
\newblock \doi{10.1137/21M1418629}.
\newblock URL \url{https://doi.org/10.1137/21M1418629}.

\bibitem[Hendrycks \& Gimpel(2016)Hendrycks and Gimpel]{Hendrycks2016-ui}
Hendrycks, D. and Gimpel, K.
\newblock {Gaussian Error Linear Units (GELUs)}.
\newblock June 2016.
\newblock URL \url{http://arxiv.org/abs/1606.08415}.

\bibitem[Hsieh et~al.(2018)Hsieh, Zhao, Eismann, Mirabella, and
  Ermon]{Hsieh2018-ey}
Hsieh, J.-T., Zhao, S., Eismann, S., Mirabella, L., and Ermon, S.
\newblock {Learning Neural PDE Solvers with Convergence Guarantees}.
\newblock In \emph{{International Conference on Learning Representations}},
  September 2018.
\newblock URL \url{https://openreview.net/pdf?id=rklaWn0qK7}.

\bibitem[Huang et~al.(2020)Huang, Wang, and Yang]{Huang2020-oa}
Huang, J., Wang, H., and Yang, H.
\newblock {Int-Deep: A deep learning initialized iterative method for nonlinear
  problems}.
\newblock \emph{Journal of computational physics}, 419:\penalty0 109675,
  October 2020.
\newblock ISSN 0021-9991.
\newblock \doi{10.1016/j.jcp.2020.109675}.
\newblock URL
  \url{https://www.sciencedirect.com/science/article/pii/S0021999120304496}.

\bibitem[Jacovi et~al.(2019)Jacovi, Hadash, Kermany, Carmeli, Lavi, Kour, and
  Berant]{Jacovi2019-pt}
Jacovi, A., Hadash, G., Kermany, E., Carmeli, B., Lavi, O., Kour, G., and
  Berant, J.
\newblock {Neural network gradient-based learning of black-box function
  interfaces}.
\newblock January 2019.
\newblock URL \url{https://openreview.net/pdf?id=r1e13s05YX}.

\bibitem[Jasak(2009)]{Jasak2009-ao}
Jasak, H.
\newblock {OpenFOAM: Open source CFD in research and industry}.
\newblock \emph{International Journal of Naval Architecture and Ocean
  Engineering}, 1\penalty0 (2):\penalty0 89--94, December 2009.
\newblock ISSN 2092-6782.
\newblock \doi{10.2478/IJNAOE-2013-0011}.
\newblock URL
  \url{https://www.sciencedirect.com/science/article/pii/S2092678216303879}.

\bibitem[Karniadakis et~al.(2021)Karniadakis, Kevrekidis, Lu, Perdikaris, Wang,
  and Yang]{Karniadakis2021-tf}
Karniadakis, G.~E., Kevrekidis, I.~G., Lu, L., Perdikaris, P., Wang, S., and
  Yang, L.
\newblock {Physics-informed machine learning}.
\newblock \emph{Nature Reviews Physics}, 3\penalty0 (6):\penalty0 422--440, May
  2021.
\newblock ISSN 2522-5820, 2522-5820.
\newblock \doi{10.1038/s42254-021-00314-5}.
\newblock URL \url{https://www.nature.com/articles/s42254-021-00314-5}.

\bibitem[Kingma \& Ba(2015)Kingma and Ba]{Kingma2015-ys}
Kingma, D.~P. and Ba, J.
\newblock {Adam: {A} Method for Stochastic Optimization}.
\newblock In \emph{{3rd International Conference on Learning Representations}},
  2015.
\newblock URL \url{http://arxiv.org/abs/1412.6980}.

\bibitem[Luz et~al.(2020)Luz, Galun, Maron, Basri, and Yavneh]{Luz2020-ky}
Luz, I., Galun, M., Maron, H., Basri, R., and Yavneh, I.
\newblock {Learning Algebraic Multigrid Using Graph Neural Networks}.
\newblock In Iii, H.~D. and Singh, A. (eds.), \emph{{Proceedings of the 37th
  International Conference on Machine Learning}}, volume 119 of
  \emph{Proceedings of Machine Learning Research}, pp.\  6489--6499. PMLR,
  2020.
\newblock URL \url{https://proceedings.mlr.press/v119/luz20a.html}.

\bibitem[Molga \& Kwietnia()Molga and Kwietnia]{Molga_undated-tx}
Molga, M. and Kwietnia, C. S.~.
\newblock {Test functions for optimization needs}.
\newblock
  \url{https://marksmannet.com/RobertMarks/Classes/ENGR5358/Papers/functions.pdf}.
\newblock URL
  \url{https://marksmannet.com/RobertMarks/Classes/ENGR5358/Papers/functions.pdf}.
\newblock Accessed: 2023-8-29.

\bibitem[Nelson(1990)]{Nelson1990-qu}
Nelson, B.~L.
\newblock {Control Variate Remedies}.
\newblock \emph{Operations research}, 38\penalty0 (6):\penalty0 974--992,
  December 1990.
\newblock ISSN 0030-364X.
\newblock \doi{10.1287/opre.38.6.974}.
\newblock URL \url{https://doi.org/10.1287/opre.38.6.974}.

\bibitem[{\"O}zbay et~al.(2021){\"O}zbay, Hamzehloo, Laizet, Tzirakis, Rizos,
  and Schuller]{Ozbay2021-qt}
{\"O}zbay, A.~G., Hamzehloo, A., Laizet, S., Tzirakis, P., Rizos, G., and
  Schuller, B.
\newblock {Poisson CNN: Convolutional neural networks for the solution of the
  Poisson equation on a Cartesian mesh}.
\newblock \emph{Data-Centric Engineering}, 2, 2021.
\newblock ISSN 2632-6736.
\newblock \doi{10.1017/dce.2021.7}.

\bibitem[Polyak(1963)]{Polyak1963-ys}
Polyak, B.~T.
\newblock {Gradient methods for the minimisation of functionals}.
\newblock \emph{USSR Computational Mathematics and Mathematical Physics},
  3\penalty0 (4):\penalty0 864--878, January 1963.
\newblock ISSN 0041-5553.
\newblock \doi{10.1016/0041-5553(63)90382-3}.
\newblock URL
  \url{https://www.sciencedirect.com/science/article/pii/0041555363903823}.

\bibitem[Ren et~al.(2022)Ren, Kornblith, Liao, and Hinton]{Ren2022-gc}
Ren, M., Kornblith, S., Liao, R., and Hinton, G.
\newblock {Scaling forward gradient with local losses}.
\newblock October 2022.
\newblock URL \url{https://github.com/google-research/}.

\bibitem[Rosenbrock(1960)]{Rosenbrock1960-lg}
Rosenbrock, H.~H.
\newblock {An Automatic Method for Finding the Greatest or Least Value of a
  Function}.
\newblock \emph{Computer Journal}, 3\penalty0 (3):\penalty0 175--184, January
  1960.
\newblock ISSN 0010-4620.
\newblock \doi{10.1093/comjnl/3.3.175}.
\newblock URL
  \url{https://academic.oup.com/comjnl/article-pdf/3/3/175/988633/030175.pdf}.

\bibitem[Saad(2003)]{Saad2003-vm}
Saad, Y.
\newblock \emph{{Iterative Methods for Sparse Linear Systems: Second Edition}}.
\newblock Other Titles in Applied Mathematics. SIAM, April 2003.
\newblock ISBN 9780898715347.
\newblock \doi{10.1137/1.9780898718003}.
\newblock URL
  \url{https://play.google.com/store/books/details?id=qtzmkzzqFmcC}.

\bibitem[Sappl et~al.(2019)Sappl, Seiler, Harders, and Rauch]{Sappl2019-kq}
Sappl, J., Seiler, L., Harders, M., and Rauch, W.
\newblock {Deep Learning of Preconditioners for Conjugate Gradient Solvers in
  Urban Water Related Problems}.
\newblock June 2019.
\newblock URL \url{http://arxiv.org/abs/1906.06925}.

\bibitem[Stolarski et~al.(2018)Stolarski, Nakasone, and
  Yoshimoto]{Stolarski2018-ly}
Stolarski, T., Nakasone, Y., and Yoshimoto, S.
\newblock \emph{{Engineering Analysis with ANSYS Software}}.
\newblock Butterworth-Heinemann, January 2018.
\newblock ISBN 9780081021651.
\newblock URL
  \url{https://play.google.com/store/books/details?id=50IyDwAAQBAJ}.

\bibitem[St{\"u}ben(2001)]{Stuben2001-sj}
St{\"u}ben, K.
\newblock {A review of algebraic multigrid}.
\newblock \emph{Journal of computational and applied mathematics}, 128\penalty0
  (1):\penalty0 281--309, March 2001.
\newblock ISSN 0377-0427.
\newblock \doi{10.1016/S0377-0427(00)00516-1}.
\newblock URL
  \url{https://www.sciencedirect.com/science/article/pii/S0377042700005161}.

\bibitem[Takamoto et~al.(2022)Takamoto, Praditia, Leiteritz, MacKinlay,
  Alesiani, Pfl{\"u}ger, and Niepert]{Takamoto2022-mr}
Takamoto, M., Praditia, T., Leiteritz, R., MacKinlay, D., Alesiani, F.,
  Pfl{\"u}ger, D., and Niepert, M.
\newblock {PDEBENCH: An extensive benchmark for Scientific machine learning}.
\newblock pp.\  1596--1611, October 2022.
\newblock URL
  \url{https://proceedings.neurips.cc/paper_files/paper/2022/file/0a9747136d411fb83f0cf81820d44afb-Paper-Datasets_and_Benchmarks.pdf}.

\bibitem[Thuerey et~al.(2021)Thuerey, Holl, Mueller, Schnell, Trost, and
  Um]{Thuerey2021-xr}
Thuerey, N., Holl, P., Mueller, M., Schnell, P., Trost, F., and Um, K.
\newblock \emph{{Physics-based Deep Learning}}.
\newblock WWW, 2021.
\newblock URL \url{https://physicsbaseddeeplearning.org}.

\bibitem[Um et~al.(2020)Um, Brand, Fei, Holl, and Thuerey]{Um2020-zx}
Um, K., Brand, R., Fei, Y.~r., Holl, P., and Thuerey, N.
\newblock {Solver-in-the-loop: learning from differentiable physics to interact
  with iterative PDE-solvers}.
\newblock In \emph{{Proceedings of the 34th International Conference on Neural
  Information Processing Systems}}, number Article 513 in NIPS'20, pp.\
  6111--6122, Red Hook, NY, USA, December 2020. Curran Associates Inc.
\newblock ISBN 9781713829546.
\newblock URL \url{https://dl.acm.org/doi/abs/10.5555/3495724.3496237}.

\bibitem[Vaupel et~al.(2020)Vaupel, Hamacher, Caspari, Mhamdi, Kevrekidis, and
  Mitsos]{Vaupel2020-zu}
Vaupel, Y., Hamacher, N.~C., Caspari, A., Mhamdi, A., Kevrekidis, I.~G., and
  Mitsos, A.
\newblock {Accelerating nonlinear model predictive control through machine
  learning}.
\newblock \emph{Journal of process control}, 92:\penalty0 261--270, August
  2020.
\newblock ISSN 0959-1524.
\newblock \doi{10.1016/j.jprocont.2020.06.012}.
\newblock URL
  \url{https://www.sciencedirect.com/science/article/pii/S0959152420302481}.

\bibitem[Venkataraman \& Amos(2021)Venkataraman and Amos]{Venkataraman2021-zv}
Venkataraman, S. and Amos, B.
\newblock {Neural Fixed-Point Acceleration for Convex Optimization}.
\newblock In \emph{{8th ICML Workshop on Automated Machine Learning (AutoML)}},
  2021.
\newblock URL \url{https://openreview.net/forum?id=Vxbpb6XvGgH}.

\bibitem[Vinuesa \& Brunton(2021)Vinuesa and Brunton]{Vinuesa2021-id}
Vinuesa, R. and Brunton, S.~L.
\newblock {The Potential of Machine Learning to Enhance Computational Fluid
  Dynamics}.
\newblock October 2021.
\newblock URL \url{http://arxiv.org/abs/2110.02085}.

\bibitem[Willard et~al.(2022)Willard, Jia, Xu, Steinbach, and
  Kumar]{Willard2022-vs}
Willard, J., Jia, X., Xu, S., Steinbach, M., and Kumar, V.
\newblock {Integrating Scientific Knowledge with Machine Learning for
  Engineering and Environmental Systems}.
\newblock \emph{ACM Comput. Surv.}, 55\penalty0 (4):\penalty0 1--37, November
  2022.
\newblock ISSN 0360-0300.
\newblock \doi{10.1145/3514228}.
\newblock URL \url{https://doi.org/10.1145/3514228}.

\bibitem[Williams(1992)]{Williams1992-ll}
Williams, R.~J.
\newblock {Simple statistical gradient-following algorithms for connectionist
  reinforcement learning}.
\newblock \emph{Machine learning}, 8\penalty0 (3):\penalty0 229--256, May 1992.
\newblock ISSN 0885-6125.
\newblock \doi{10.1007/BF00992696}.
\newblock URL \url{https://doi.org/10.1007/BF00992696}.

\end{thebibliography}
\bibliographystyle{icml2024}

\newpage
\appendix
\onecolumn
\section{Appendix}
\subsection{implementation}
The source code of the experiments is available at \url{https://github.com/icml2024paper/NI-GBMS}.

\subsection{Proofs}
\label{sec:proofs}

\begin{proof}[Proof of \cref{thm:property}]
    To lighten the notation, we omit $\vtheta$ dependence.
    Denoting $i$th component of $\vg_\rvv$ by ${g_{\rvv}}_i$,
    \begin{align}
        \E [{{g_{\rvv}}_i}^2]
         & = \E \left[(\nabla f \cdot \rvv)^2 \rv_i^2 \right]                                                                              \\
         & = \E \left[ \left(\pder[f]{\theta_i}\right)^2 {\rv_i}^4 + \sum_{j\neq i}\left(\pder[f]{\theta_j}\right)^2 {\rv_i}^2 {\rv_j}^2
        + 2 \sum_{k < l} \pder[f]{\theta_k}\pder[f]{\theta_l}{\rv_i}^2\rv_k\rv_l \right]                                                   \\
         & = \left(\pder[f]{\theta_i}\right)^2 \E[ {\rv_i}^4] + \sum_{j\neq i}\left(\pder[f]{\theta_j}\right)^2 \E[{\rv_i}^2]\E[{\rv_j}^2]
        + 2 \sum_{k < l} \pder[f]{\theta_k}\pder[f]{\theta_l} \E[{\rv_i}^2]\E[\rv_k]\E[\rv_l]                                              \\
         & = \left(\pder[f]{\theta_i}\right)^2 (1 + \Var[{\rv_i}^2]) + \sum_{j\neq i}\left(\pder[f]{\theta_j}\right)^2                     \\
         & = \left(\pder[f]{\theta_i}\right)^2 \Var[{\rv_i}^2] + \| \nabla f\|^2.
    \end{align}
    Hence,
    \begin{align}
        \E [{{h_{\rvv}}_i}^2]
         & = \E\left[ \left( {g_{\rvv}}_i - \hat{g_\rvv}_i + \E[\hat{g_\rvv}_i]\right)^2\right]                                    \\
         & = \E\left[ \left((\nabla f \cdot \rvv)\rv_i - (\nabla \hat f \cdot \rvv)\rv_i + \pder[\hat f]{\theta_i}\right)^2\right] \\
         & = \E\left[ \left((\nabla f - \nabla \hat f) \cdot \rvv)\rv_i  + \pder[\hat f]{\theta_i}\right)^2\right]                 \\
         & = \E\left[ \left((\nabla f - \nabla \hat f) \cdot \rvv\right)^2 \rv_i^2
            + 2 \left((\nabla f - \nabla \hat f) \cdot \rvv\right)\rv_i \pder[\hat f]{\theta_i}
        + \left(\pder[\hat f]{\theta_i}\right)^2 \right]                                                                           \\
         & = \E\left[ \left((\nabla f - \nabla \hat f) \cdot \rvv\right)^2 \rv_i^2\right]
        + 2 \E \left[((\nabla f - \nabla \hat f) \cdot \rvv)\rv_i\right] \pder[\hat f]{\theta_i}
        + \left(\pder[\hat f]{\theta_i}\right)^2                                                                                   \\
         & = \left(\pder[f]{\theta_i} - \pder[\hat f]{\theta_i}\right)^2 \Var[{\rv_i}^2] + \|\nabla f - \nabla \hat f\|^2
        + 2 \left(\pder[f]{\theta_i} - \pder[\hat f]{\theta_i}\right)\pder[\hat f]{\theta_i}
        + \left(\pder[\hat f]{\theta_i}\right)^2                                                                                   \\
         & = \left(\pder[f]{\theta_i} - \pder[\hat f]{\theta_i}\right)^2 \Var[{\rv_i}^2] + \|\nabla f - \nabla \hat f\|^2
        + 2\pder[f]{\theta_i}\pder[\hat f]{\theta_i} - \left(\pder[\hat f]{\theta_i}\right)^2.
    \end{align}
    Therefore,
    \begin{align}
          & \E[\|\vh_\rvv - \nabla f\|^2]                                                            \\
        = & \E[\|\vh_\rvv\|^2] - \|\nabla f\|^2                                                      \\
        = & \sum_{i=1}^d \E [{{h_{\rvv}}_i}^2] - \|\nabla f\|^2                                      \\
        = & \sum_{i=1}^d \left(\pder[f]{\theta_i} - \pder[\hat f]{\theta_i}\right)^2 \Var[{\rv_i}^2]
        + d\|\nabla f - \nabla \hat f\|^2
        + 2 \nabla f \cdot \nabla \hat f
        - \|\nabla \hat f\|^2
        - \|\nabla f\|^2                                                                             \\
        = & \sum_{i=1}^d \left(\pder[f]{\theta_i} - \pder[\hat f]{\theta_i}\right)^2 \Var[{\rv_i}^2]
        + (d-1)\|\nabla f - \nabla \hat f\|^2
    \end{align}
    This is minimized when $\Var[{\rv_i}^2]=0$ for all $i$, which implies $\rv_i$'s are independent Rademacher variables.
    Then, the minimized deviation is
    \begin{equation}
        \E[\|\vh_\rvv - \nabla f\|^2] = (d-1)\|\nabla f - \nabla \hat f\|^2.
    \end{equation}
\end{proof}

\begin{proof}[Proof of \cref{thm:convergence}]
    To prove \cref{thm:convergence}, we introduce \cref{lem:smooth} and \cref{lem:convex}.

    \begin{lemma}
        \label{lem:smooth}
        If $f:\R^d\to\R$ is $L$-smooth, then for all $k \in \N$, we have
        \begin{equation}
            \label{eq:smooth}
            \E[f(\vtheta_{k+1})] - f(\vtheta_k)
            \leq -\alpha \|\nabla f(\vtheta_k)\|^2 + \frac{\alpha^2 Ld}{2} \|\nabla f(\vtheta_k) - \nabla \hat f_k(\vtheta_k) \|^2.
        \end{equation}
    \end{lemma}
    \begin{proof}
        Since $f$ is $L$-smooth, we have
        \begin{align}
            f(\vtheta_{k+1}) - f(\vtheta_k) & \leq \nabla f(\vtheta_k) \cdot (\vtheta_{k+1} - \vtheta_{k}) + \frac{L}{2} \|\vtheta_{k+1} - \vtheta_{k} \| \\
                                            & = \alpha \nabla f(\vtheta_k) \cdot \vh_\rvv(\vtheta_k) + \frac{\alpha^2 L}{2} \|\vh_\rvv(\vtheta_k) \|
        \end{align}
        Then, taking expectations with respect to $\rvv$ at the $k$th step and using \cref{eq:expectation} and \cref{eq:variance},
        we have the desired inequality.
    \end{proof}

    \begin{lemma}[\citet{Polyak1963-ys}]
        \label{lem:convex}
        $\mu$-strong convexity implies $\mu$-Polyak-Łojasiewicz inequality:
        \begin{equation}
            2\mu(f(\vtheta) - f(\vtheta_*)) \leq \|\nabla f(\vtheta)\|^2
        \end{equation}
    \end{lemma}
    \begin{proof}
        The proof can be found in \citet{Bottou2016-zs}.
    \end{proof}
    Now let us prove \cref{thm:convergence}.
    \begin{enumerate} [label=(\alph*)]
        \item \label{proof:a}
              When $\hat f_k\equiv0$, \cref{eq:smooth} reduces to
              \begin{align}
                  \E[f(\vtheta_{k+1})] - f(\vtheta_k)
                   & \leq -\alpha \|\nabla f(\vtheta_k)\|^2 + \frac{\alpha^2 Ld}{2} \|\nabla f(\vtheta_k) \|^2 \\
                   & = - \alpha (1 - \frac{\alpha L d}{2}) \|\nabla f (\vtheta_k) \|^2                         \\
                   & \leq - \frac{\alpha}{2} \|\nabla f (\vtheta_k) \|^2                                       \\
                   & \leq - \alpha \mu (f(\vtheta_k) - f(\vtheta_*)).
              \end{align}

              Subtracting $f(\vtheta_*)$, taking total expectations, and rearranging, this yields
              \begin{equation}
                  \E[f(\vtheta_{k+1}) - f(\vtheta_*)] \leq (1-\alpha \mu)\E[f(\vtheta_{k}) - f(\vtheta_*)].
              \end{equation}
              By applying this inequality repeatedly, the desired inequality follows.

        \item When $\sup_{\vtheta\in\R^d}\{ \frac{\|\nabla f(\vtheta) - \nabla \hat f(\vtheta)\|^2}{\|\nabla f(\vtheta)\|^2}\} \leq r$,
              \cref{eq:smooth} reduces to
              \begin{equation}
                  \E[f(\vtheta_{k+1})] - f(\vtheta_k) \leq -\alpha \|\nabla f(\vtheta_k)\|^2 + \frac{\alpha^2 Ldr}{2} \|\nabla f(\vtheta_k) \|^2.
              \end{equation}
              Then, the rest of the proof is the same as \ref{proof:a}

        \item By \cref{lem:convex}, \cref{eq:smooth} gives
              \begin{equation}
                  \E[f(\vtheta_{k+1})] - f(\vtheta_k) \leq -2\alpha \mu (f(\vtheta_k) - f(\vtheta_*))  + \frac{\alpha^2 Ld}{2} \|\nabla f(\vtheta_k) - \nabla \hat f_k(\vtheta_k) \|^2.
              \end{equation}
              Subtracting $f(\vtheta_*)$, taking total expectations, and rearranging, this yields
              \begin{equation}
                  \E[f(\vtheta_{k+1}) - f(\vtheta_*)] \leq (1-2\alpha \mu)\E[f(\vtheta_{k}) - f(\vtheta_*)]
                  + \frac{\alpha^2 Ld}{2} \E[\|\nabla f(\vtheta_k) - \nabla \hat f_k(\vtheta_k) \|^2].
              \end{equation}
              Using the assumption \cref{eq:assumption}, we have
              \begin{align}
                \E[f(\vtheta_{k+1}) - f(\vtheta_*)]
                 & \leq (1-2\alpha \mu)\E[f(\vtheta_{k}) - f(\vtheta_*)]
                + \frac{\alpha^2 Ldr}{2} \xi^k \|\nabla f(\vtheta_0) - \nabla \hat f_0(\vtheta_0)\|^2 \\
                 & \leq (1-2\alpha \mu)\E[f(\vtheta_{k}) - f(\vtheta_*)]
                + \alpha \mu C \xi^k,
            \end{align}
            where
            \begin{equation}
                C = \max\{\frac{\alpha L d \|\nabla f(\vtheta_0) - \nabla \hat f_0(\vtheta_0)\|^2}{\mu}, f(\vtheta_0) - f(\vtheta_*) \}.
            \end{equation}
            Using this inequality, we prove \cref{eq:online} by induction. For $k=0$, we have
            \begin{align}
                \E[f(\vtheta_{1}) - f(\vtheta_*)]
                 & \leq (1-2\alpha \mu)(f(\vtheta_{0}) - f(\vtheta_*)) +
                \alpha \mu C                                             \\
                 & \leq \rho C.
            \end{align}
            Assume \cref{eq:online} holds for $k$. Then, we have
            \begin{align}
                \E[f(\vtheta_{k+1}) - f(\vtheta_*)]
                 & \leq (1-2\alpha \mu)C\rho^k
                + \alpha \mu C \xi^k                                                                         \\
                 & \leq C \rho^k \left( 1 - 2 \alpha \mu + \alpha \mu \left(\frac{\xi}{\rho}\right)^k\right) \\
                 & \leq C \rho^k \left( 1 - 2 \alpha \mu + \alpha \mu \right)                                \\
                 & \leq C \rho^{k+1}.
            \end{align}

    \end{enumerate}
\end{proof}

\subsection{Analysis for the finite difference approximation}
\label{app:finite difference}

\begin{theorem}
    Let $\rv_i$'s be independent Rademacher variables and $f:\R^d\to\R$ be $L$-smooth.
    Then, we have the following bounds for the control variate forward gradient
    $\vh_{\rvv, \epsilon}(\vtheta)$ for all $\vtheta \in \R^d$:
    \begin{equation}
        \label{thm:hv mean}
        \norm{\E[\vh_{\rvv, \epsilon}(\vtheta)] - \nabla f(\vtheta)} \leq \frac{\epsilon L d^{3/2}}{2},
    \end{equation}
    and
    \begin{equation}
        \label{thm:hv var}
        \E[ \norm{\vh_{\rvv, \epsilon}(\vtheta) - \nabla f(\vtheta)}^2] 
        \leq \frac{\epsilon^2 L^2 d^3}{4} + 2 (d-1) \frac{\epsilon L d^{3/2}}{2} \norm{\nabla f(\vtheta) - \nabla \hat f(\vtheta)} + (d-1) \norm{\nabla f(\vtheta) - \nabla \hat f(\vtheta)}^2\\
    \end{equation}
\end{theorem}

\begin{proof}
    Since $f$ is $L$-smooth, we have
    \begin{align}
        \lvert f(\vx) - f(\vy) - \nabla f(\vy) \cdot (\vx - \vy) \rvert \leq \frac{L}{2} \norm{\vx - \vy}^2 
        \quad (\forall \vx, \vy \in \R^d).
    \end{align}
    This inequality with $\vtheta + \epsilon \rvv$ and $\vtheta$ yields
    \begin{align}
        \lvert f(\vtheta + \epsilon \rvv) - f(\vtheta) - \nabla f(\vtheta) \cdot \epsilon \rvv \rvert
        & \leq \frac{L}{2} \norm{\epsilon \rvv}^2 \\
        & = \frac{\epsilon^2 L d}{2}.
    \end{align}
    We denote $D_{\rvv, \epsilon}(\vtheta) := \frac{f(\vtheta + \epsilon \rvv) - f(\vtheta)}{\epsilon} - \nabla f(\vtheta) \cdot \rvv$.
    Note that $\lvert D_{\rvv, \epsilon}(\vtheta) \rvert \leq \epsilon L d / 2$.
    Using this inequality, we can show inequality (\ref{thm:hv mean}) as follows:
    \begin{align}
        \norm{\E[\vh_{\rvv, \epsilon}(\vtheta)] - \nabla f(\vtheta)}
        & = \norm{ \E[\vg_{\rvv,\epsilon}(\vtheta) - \hat{\vg_\rvv}(\vtheta) + \E[\hat{\vg_\rvv}(\vtheta)]] 
        - \E[\vg_\rvv(\vtheta)]} \\
        & = \norm{ \E[\vg_{\rvv,\epsilon}(\vtheta) - {\vg_\rvv}(\vtheta)]} \\
        & \leq \E[\norm{\vg_{\rvv,\epsilon}(\vtheta) - {\vg_\rvv}(\vtheta)}] \\
        & = \E[\norm {\frac{f(\vtheta + \epsilon \rvv) - f(\vtheta)}{\epsilon}\rvv - (\nabla f(\vtheta) \cdot \rvv) \rvv}] \\
        & = \E[ \lvert D_{\rvv, \epsilon}(\vtheta) \rvert  \norm {\rvv}] \\
        & \leq \frac{\epsilon L d}{2} d^{1/2} \\
        & = \frac{\epsilon L d^{3/2}}{2}.
    \end{align}
    Next, we show inequality (\ref{thm:hv var}). 
    To simplify the notation, we omit the argument $\vtheta$. 
    Then, we have
    \begin{align}
        \E[\norm{\vh_{\rvv, \epsilon} - \nabla f}^2] 
        & = \E[\norm{ (\vg_{\rvv,\epsilon} - \vg_\rvv) + (\vh_\rvv - \nabla f)}^2] \\
        & = \E[\norm{ D_{\rvv, \epsilon} \rvv + (\vh_\rvv - \nabla f)}^2] \\
        & = \E[\norm{ D_{\rvv, \epsilon} \rvv}^2] + 2 \E [D_{\rvv, \epsilon} \rvv \cdot (\vh_\rvv - \nabla f)] + \E[\norm{\vh_\rvv - \nabla f}^2]\\
        & = \E[ {D_{\rvv, \epsilon}}^2 \norm{\rvv}^2] + 2 \E [D_{\rvv, \epsilon} \rvv \cdot (\vg_\rvv - \hat \vg_\rvv + \nabla \hat f - \nabla f)] + (d-1) \norm{\nabla f - \nabla \hat f}^2\\
        & = \E[ {D_{\rvv, \epsilon}}^2 \norm{\rvv}^2] + 2 \E [D_{\rvv, \epsilon} \rvv \cdot (\vg_\rvv - \hat \vg_\rvv)] 
        + 2 \E[D_{\rvv, \epsilon} \rvv] \cdot (\nabla \hat f - \nabla f) + (d-1) \norm{\nabla f - \nabla \hat f}^2\\
        & = \E[ {D_{\rvv, \epsilon}}^2 \norm{\rvv}^2] + 2 \E [D_{\rvv, \epsilon} \rvv \cdot ((\nabla f - \nabla \hat f) \cdot \rvv) \rvv] 
        + 2 \E[D_{\rvv, \epsilon} \rvv] \cdot (\nabla \hat f - \nabla f) + (d-1) \norm{\nabla f - \nabla \hat f}^2\\
        & = \E[ {D_{\rvv, \epsilon}}^2 \norm{\rvv}^2] + 2 \E [D_{\rvv, \epsilon}((\nabla f - \nabla \hat f) \cdot \rvv) \norm{\rvv}^2] 
        + 2 \E[D_{\rvv, \epsilon} \rvv] \cdot (\nabla \hat f - \nabla f) + (d-1) \norm{\nabla f - \nabla \hat f}^2\\
        & = d \E[ {D_{\rvv, \epsilon}}^2] + 2 d \E [D_{\rvv, \epsilon}((\nabla f - \nabla \hat f) \cdot \rvv)] 
        + 2 \E[D_{\rvv, \epsilon} \rvv] \cdot (\nabla \hat f - \nabla f) + (d-1) \norm{\nabla f - \nabla \hat f}^2\\
        & = d \E[ {D_{\rvv, \epsilon}}^2] + 2 (d-1) \E [D_{\rvv, \epsilon} \rvv] \cdot  (\nabla f - \nabla \hat f) + (d-1) \norm{\nabla f - \nabla \hat f}^2\\
        & \leq d (\frac{\epsilon L d}{2})^2 + 2 (d-1) \frac{\epsilon L d^{3/2}}{2} \norm{\nabla f - \nabla \hat f} + (d-1) \norm{\nabla f - \nabla \hat f}^2\\
        & = \frac{\epsilon^2 L^2 d^3}{4} + 2 (d-1) \frac{\epsilon L d^{3/2}}{2} \norm{\nabla f - \nabla \hat f} + (d-1) \norm{\nabla f - \nabla \hat f}^2.
    \end{align}
\end{proof}

\subsection{Details of numerical examples}
\label{app:details}
\subsubsection{Details of \cref{sec:numerical}}
\paragraph{Surrogate model}
In \cref{sec:numerical}, our surrogate model $\hat f$ has two convlutional layers with 64 filters, 
followed by a global average pooling and  a fully connected layer with 64 units.
The kernel size is 1 for the Sphere function and 3 for the Rosenbrock function.
The activation function is GELU \citep{Hendrycks2016-ui}.

\paragraph{Results}
The standard deviation of the objective values are presented in \cref{tab:basic_std},
and convergence plots are shown in \cref{fig:convergence all}.

\begin{figure}[tb]
    \centering
    \begin{subfigure}[b]{0.38\textwidth}
       \centering
       \includegraphics[width=\textwidth]{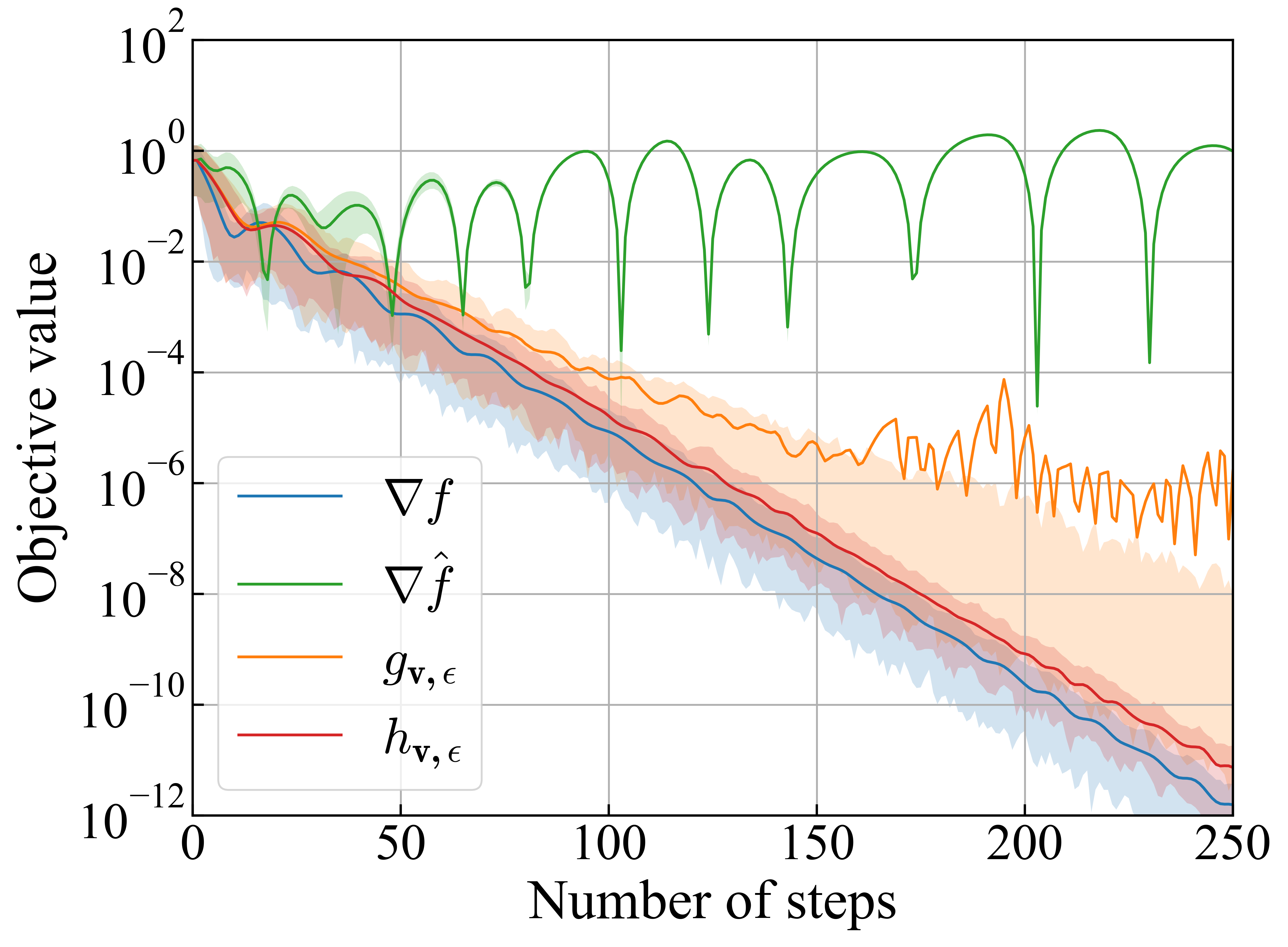}
       \caption{sphere2D}
    \end{subfigure}
    \begin{subfigure}[b]{0.38\textwidth}
       \centering
       \includegraphics[width=\textwidth]{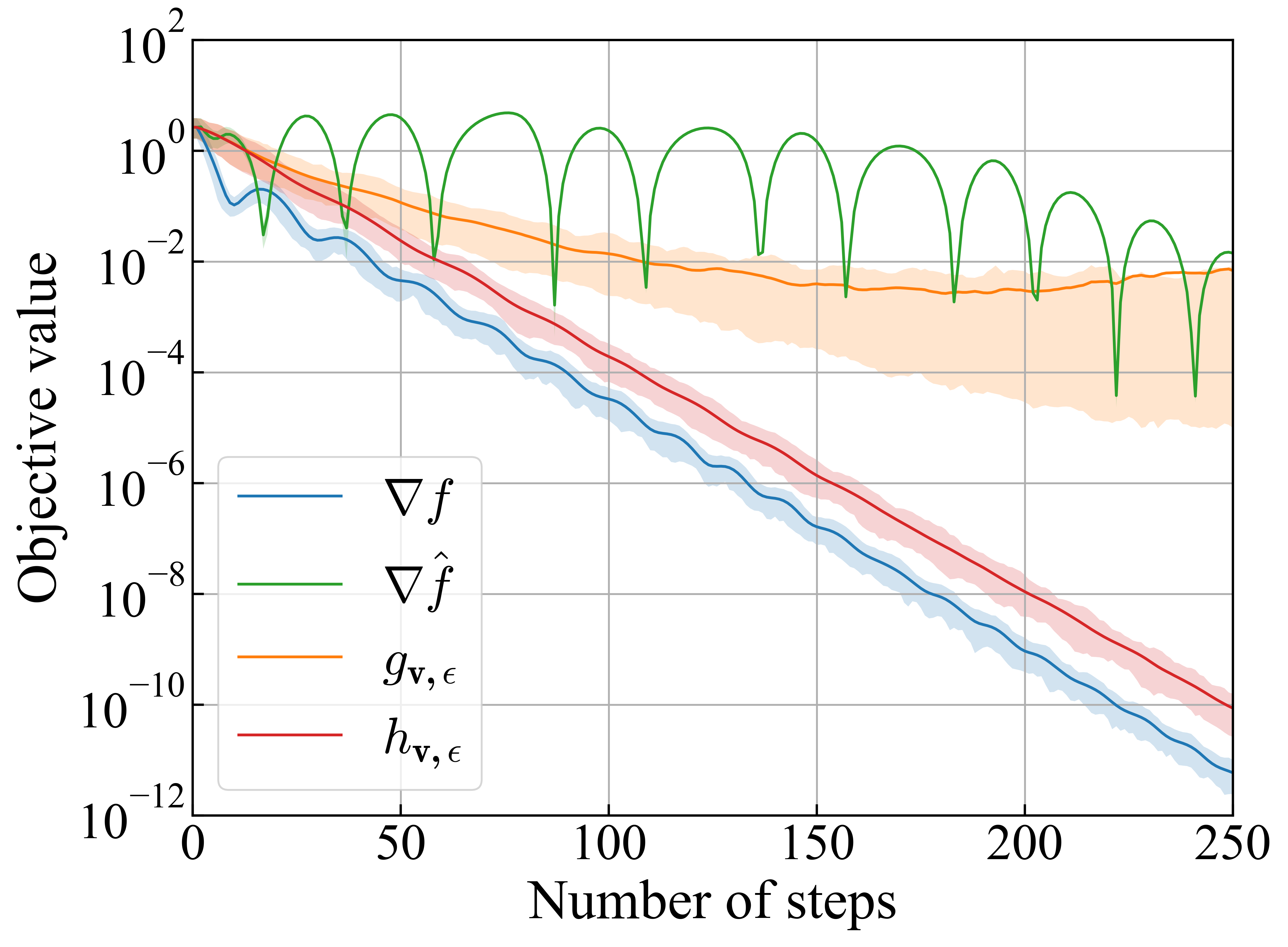}
       \caption{sphere8D}
    \end{subfigure}\\
    \begin{subfigure}[b]{0.38\textwidth}
       \centering
       \includegraphics[width=\textwidth]{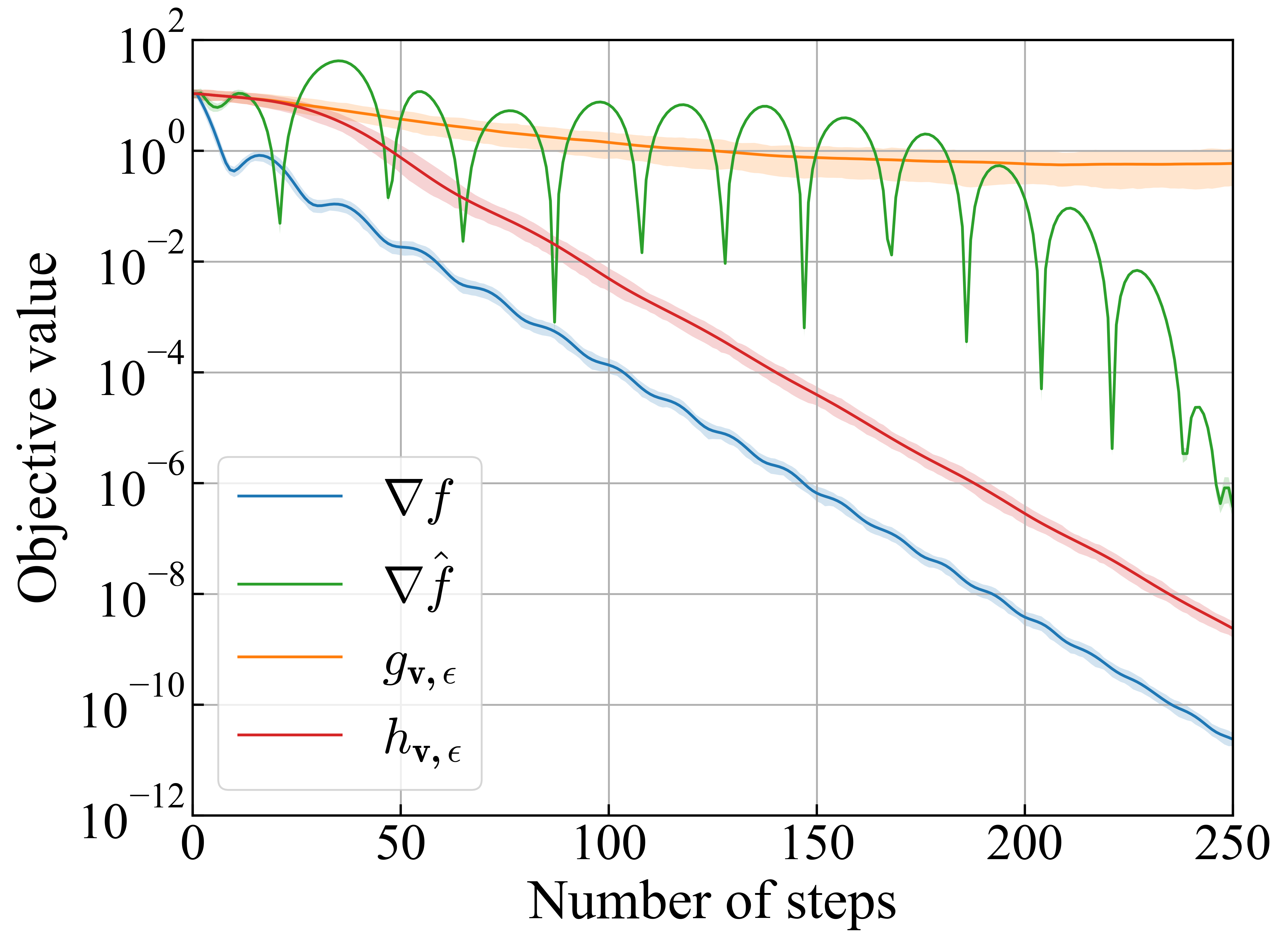}
       \caption{sphere32D}
    \end{subfigure}
    \begin{subfigure}[b]{0.38\textwidth}
       \centering
       \includegraphics[width=\textwidth]{figures/sphere128D.png}
       \caption{sphere128D}
    \end{subfigure}\\
    \begin{subfigure}[b]{0.38\textwidth}
       \centering
       \includegraphics[width=\textwidth]{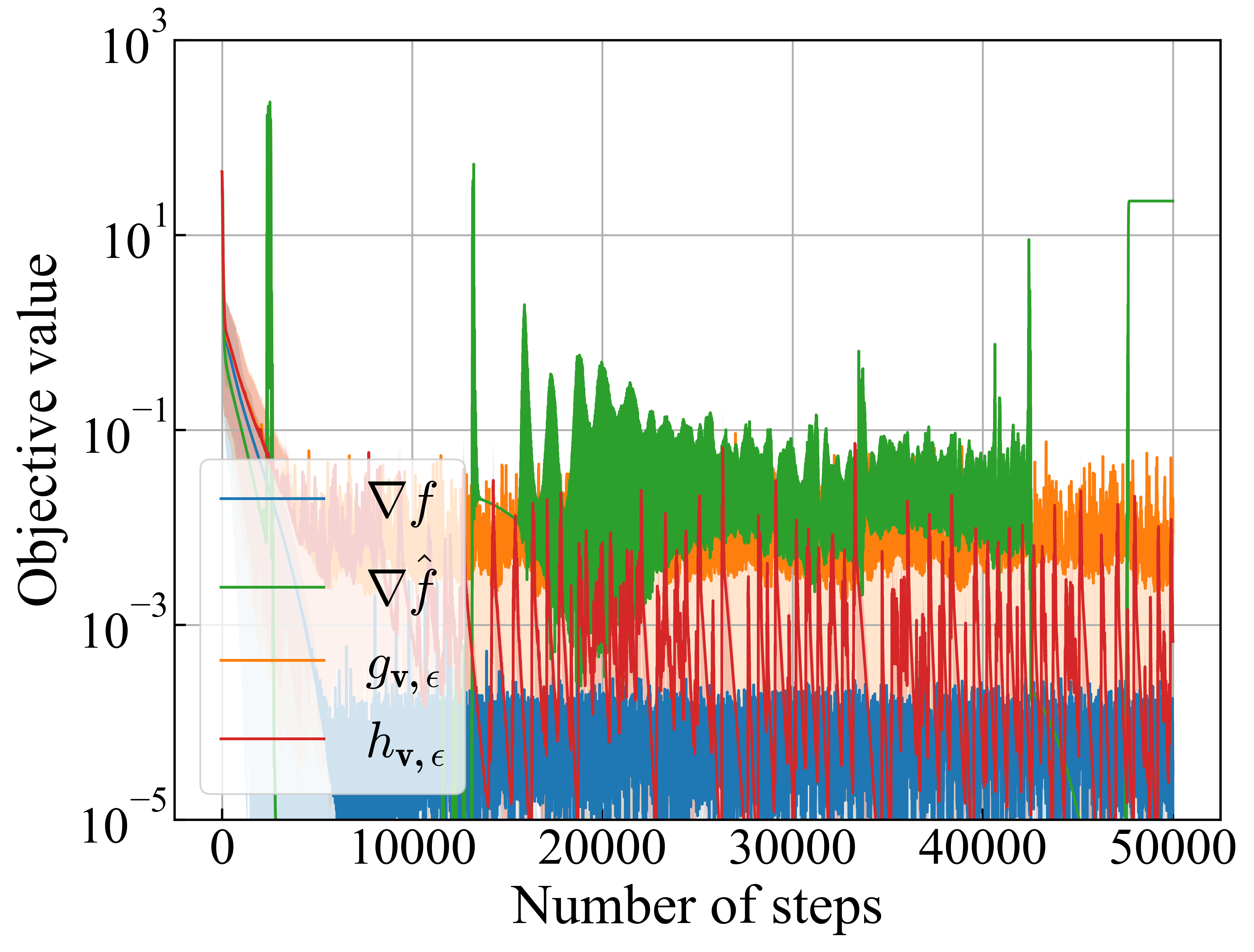}
       \caption{rosenbrock2D}
    \end{subfigure}
    \begin{subfigure}[b]{0.38\textwidth}
       \centering
       \includegraphics[width=\textwidth]{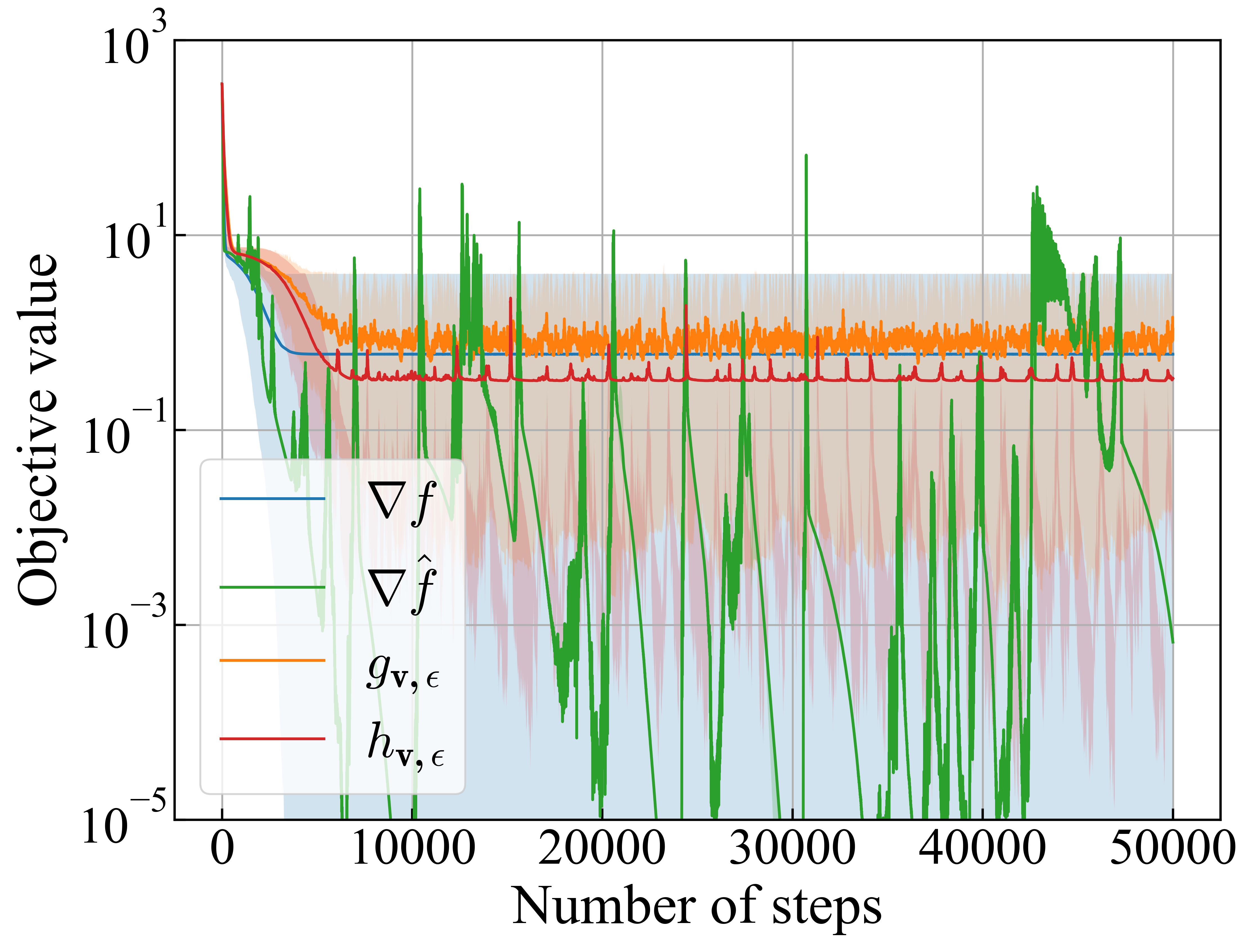}
       \caption{rosenbrock8D}
    \end{subfigure}\\
    \begin{subfigure}[b]{0.38\textwidth}
       \centering
       \includegraphics[width=\textwidth]{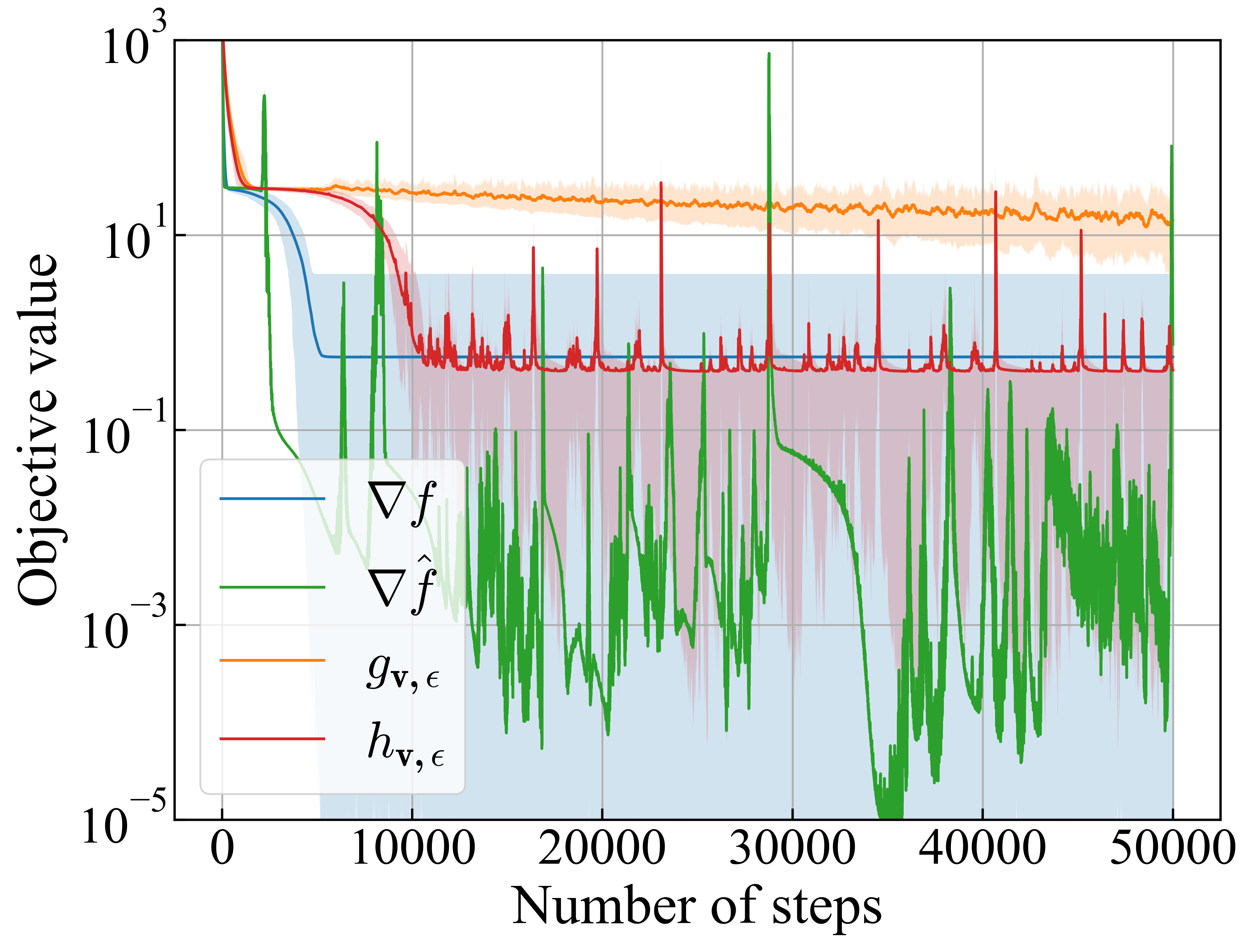}
       \caption{rosenbrock32D}
    \end{subfigure}
    \begin{subfigure}[b]{0.38\textwidth}
       \centering
       \includegraphics[width=\textwidth]{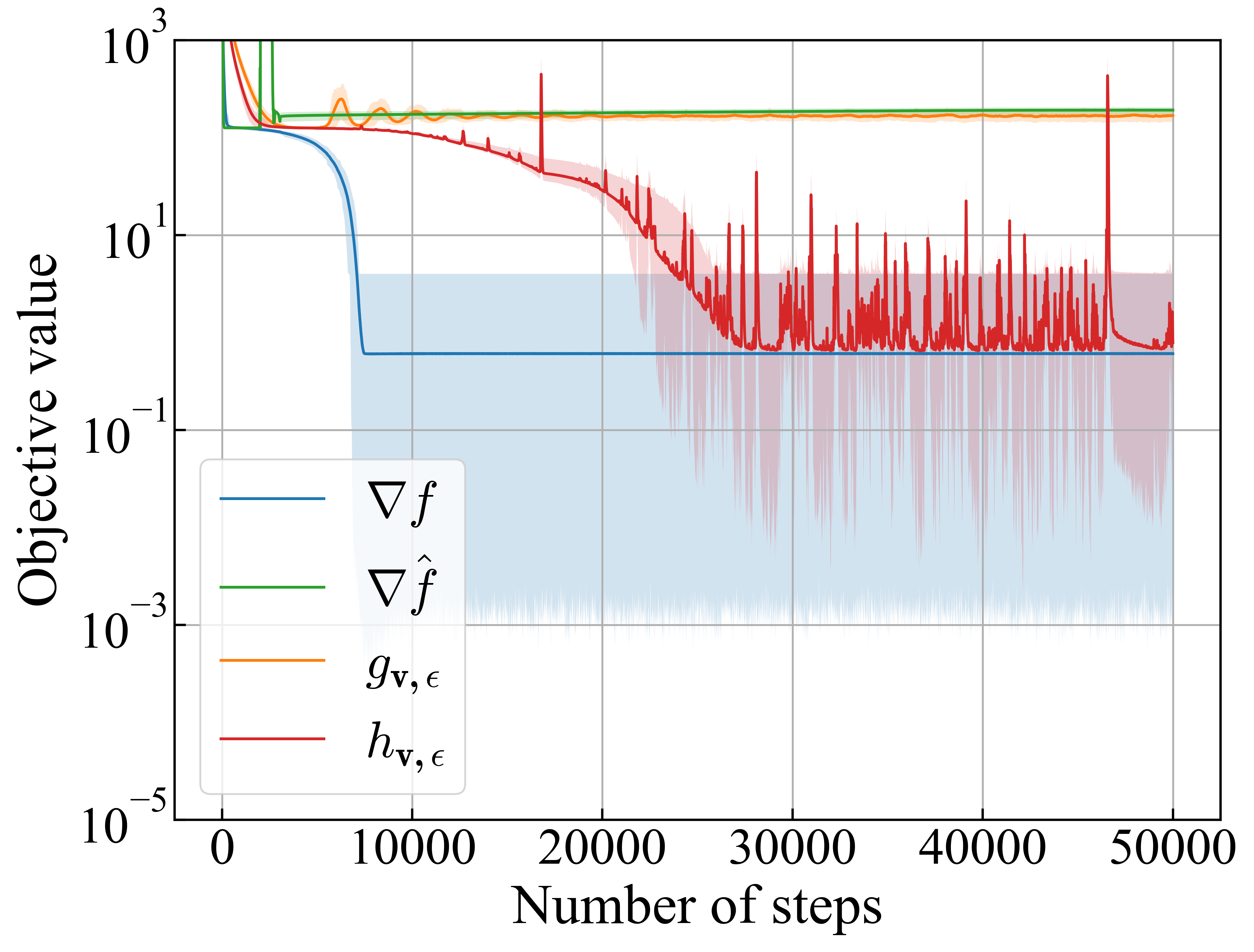}
       \caption{rosenbrock128D}
    \end{subfigure}\\
    \caption{Convergence plots. The horizontal axis is the number of Adam steps, and the vertical axis is the objective value.}
    \label{fig:convergence all}
 \end{figure}

 \begin{table*}[b]
    \caption{The standard deviations of objective values obtained using different gradient estimators.}
    \label{tab:basic_std}
    \begin{subtable}[t]{0.45\textwidth}
    \centering
    \caption{Sphere function}
    \label{tab:sphere_std}
    \begin{small}
    \begin{tabular}{l|rrrr}
        \toprule
        gradient                        & $d=2$    & $d=8$    & $d=32$   & $d=128$  \\
        \midrule
          $\nabla f$                    & 1.61e-12 & 3.02e-12 & 5.23e-12 & 1.04e-11 \\
          $\nabla \hat f$               & 2.09e-03 & 7.38e-06 & 1.36e-07 & 1.88e-04  \\
          $\vg_{\rvv,\epsilon}$         & 1.57e-05 & 2.71e-02 & 3.44e-01 & 5.86e+00 \\
          $\vh_{\rvv,\epsilon}$ (ours)  & 9.11e-12 & 7.28e-11 & 6.42e-10 & 1.17e-07 \\
          \bottomrule
       \end{tabular}
    \end{small}
    \end{subtable}
    \hfill
    \begin{subtable}[t]{0.45\textwidth}
    \centering
    \caption{Rosenbrock function}
    \label{tab:sphere_std}
    \begin{small}
    \begin{tabular}{l|rrrr}
        \toprule
        gradient                        & $d=2$    & $d=8$ & $d=32$   & $d=128$ \\
        \midrule
          $\nabla f$                    & 4.49e-05  & 1.42e+00 & 1.38e+00    & 1.43e+00   \\
          $\nabla \hat f$               & 1.44e-01  & 9.66e-07  & 3.77e-02 & 6.72e+01     \\
        $\vg_{\rvv,\epsilon}$                    & 3.59e-02 & 3.47e+00  & 9.85e+00     & 1.94e+01     \\
          $\vh_{\rvv,\epsilon}$ (ours)  & 9.95e-04 & 1.08e+00 & 1.20e+00    & 1.52e+00   \\
          \bottomrule
        \end{tabular}
    \end{small}
    \end{subtable}
 \end{table*}

 \subsubsection{Details of \cref{sec:poisson}}

 \paragraph{Task}
 In both of $P$ and $Q$, solution $u_\tau$ has the following form:
\begin{equation}
    \label{eq:u}
    u_\tau(z) = \sum_{i=1}^{N}c_i \sin{ i  \pi  z},
\end{equation}
where $c_i$'s are sampled from different distributions.
The task distribution $P$ is a mixture of two distributions $P_1$ and $P_2$ with weight $p$, i.e. $P=pP_1+(1-p)P_2$,
and $p$ is set to 0.01 in the experiment.
In distribution $P_1$, we have $c_i \sim \mathcal{N}(0, \left|\frac{N+1-2i}{N-1}\right|)$.
In distribution $P_2$, we have $c_i \sim \mathcal{N}(0, 1 - \left|\frac{N+1-2i}{N-1}\right|)$.
Similarly, the task distribution $Q$ is a mixture of two distributions $Q_1$ and $Q_2$ with weight $q$, i.e. $Q=qQ_1+(1-q)Q_2$,
and $q$ is set to 0.01 in the experiment.
In distribution $Q_1$, we have $c_i \sim \mathcal{N}(0, 1 - \frac{i}{N})$.
In distribution $Q_2$, we have $c_i \sim \mathcal{N}(0, \frac{i}{N})$.
We use $P$ for the Jacobi method and $Q$ for the multigrid method.
They are designed to contain a small number of difficult tasks, which make learning difficult \cite{Arisaka2023-aw}.
The discretization size is $N=31$ in the experiments.
The number of tasks for training, validation, and test are all $5,000$.

\paragraph{Solver}
The solver $f_{\text{MG}}$ is the simplest two-level multigrid solver with the linear prolongation 
\begin{equation}
    I_{2 h}^h=\frac{1}{2}\left(\begin{array}{ccccc}
    1 & & & & \\
    2 & & & & \\
    1 & 1 & & & \\
    & 2 & & & \\
    & 1 & 1 & & \\
    & & 2 & \ddots & \\
    & & 1 & \ddots & \\
    & & & & 1 \\
    & & & 2 \\
    & & & & 1
    \end{array}\right) \in \mathbb{R}^{(N-1) \times(N / 2-1)}
    \end{equation}
and weighted restriction
\begin{equation}
I_h^{2 h}=\frac{1}{4}\left(\begin{array}{lllllllllllll}
1 & 2 & 1 & & & & & & & & \\
& & 1 & 2 & 1 & & & & & & \\
& & & & 1 & 2 & 1 & & & & \\
& & & & & & & \ddots & & & \\
& & & & & & & & & \\
& & & & & & & & 1 & 2 & 1
\end{array}\right) \in \mathbb{R}^{(N / 2-1) \times(N-1)}.
\end{equation}
In each level, we conduct 4 iterations of the weighted Jacobi method with weight $2/3$ as a smoother.

\paragraph{Meta-solver}
The meta-solver $\Psi$ is a fully-connected neural network with one hidden layer with 512 units and SiLU activation function \citep{Elfwing2018-ae}.
It takes the right-hand side $b_\tau$ as input and gives the coefficients $c_i$'s in \cref{eq:u} to construct initial guess $\theta_\tau$.

\paragraph{Surrogate model}
In \cref{sec:poisson}, our surrogate model $\hat f$ is a fully-connected neural network with 
three hidden layers with 1024 units and GELU activation function.
It takes the initial guess $\vtheta_\tau$ and the solver's output $\hat u_\tau$ as input.

\paragraph{Training}
We use the Adam optimizer for all training.
For the Jacobi method, we use learning rate $10^{-5}$ for the meta-solver and $5.0 \times 10^{-4}$ for the surrogate model.
For the multigrid method, we use learning rate $10^{-6}$ for the meta-solver and $5.0 \times 10^{-4}$ for the surrogate model.
For the supervised baseline, we use learning rate $10^{-4}$.
The batch size is set to 256 for all training, and the number of epochs is set to 100.
The best model is selected based on the validation loss.
The finite difference step size $\epsilon$ is set to $10^{-12}$ for all training.
We conducted the experiments with 3 different random seeds.

\paragraph{Results}
The standard deviations of the results \cref{tab:number of iterations} are presented in \cref{tab:number of iterations std}.

\begin{table}[tb]
    \caption{The standard deviations of the results in \cref{tab:number of iterations}.}
    \label{tab:number of iterations std}
    \centering
    \begin{small}
    \begin{tabular}{ll|rr}
        \toprule
        $\Psi$               & gradient                 & $f = f_{\mathrm{Jac}}$ & $f = f_{\mathrm{MG}}$ \\
        \midrule
        $\Psi_\mathrm{SL}$   & -                                &   1.83      &  0.86    \\
        $\Psi_{\mathrm{GBMS}}$ & $\nabla f$                     &  0.70            &  0.45        \\
                             & $\nabla \hat f$                  &  1.35        &     3.85     \\
                             & $\vg_\rvv$                       &   1.70           &     0.22    \\
                             & $\vg_{\rvv,\epsilon}$            &   1.68          &    0.19  \\
                             & $\vh_\rvv$                       &  1.09   &    0.60    \\
                             & $\vh_{\rvv,\epsilon}$ (ours)     &   0.71           &    0.62  \\
        \bottomrule
    \end{tabular}
    \end{small}
\end{table}

\subsubsection{Details of \cref{sec:advanced}}

\paragraph{Task}
The problem formulation is based on the demos in the DOLFIN documentation \cite{noauthor_undated-zc}.
As for the discretization size, we use $3 \times 3$ unit square mesh for the biharmonic problem 
and $3 \times 3 \times 3$ unit box mesh for the elasticity problem.
For the biharmonic problem, the source term $b_\tau$ has the following form:
\begin{equation}
    b_\tau(x,y) = c_1 \sin(c_2 \pi x)\sin(c_3\pi y),
\end{equation}
where 
\begin{align}
    c_1 & \sim \mathrm{LogUniform}([10^{-2}, 10^2]), \\
    c_2 & \sim \mathrm{DiscreteUniform}([1, 2, 3, 4, 5]), \\
    c_3 & \sim \mathrm{DiscreteUniform}([1, 2, 3, 4, 5]).
\end{align}

The number of tasks for training, validation, and test are all $5,000$ for both problem settings.

\paragraph{Solver}
The solver $f_{\text{AMG}}$ is constructed by setting the following PETSc options:
\begin{itemize}
    \item \texttt{-ksp\_type richardson}
    \item \texttt{-ksp\_initial\_guess\_nonzero True}
    \item \texttt{-pc\_type gamg}
\end{itemize}
Other options are set to default values.

\paragraph{Meta-solver}
The meta-solver $\Psi$ is a fully-connected neural network with one hidden layer with 1024 units and SiLU activation function 
for both biharmonic and elasticity problem.

\paragraph{Surrogate model}
Surrogate model $\hat f$ is a fully-connected neural network with three hidden layers with 512 units and GELU activation function
for both biharmonic and elasticity problem.
It takes the initial guess $\vtheta_\tau$ and the solver's solution as input.

\paragraph{Training}
We use the Adam optimizer for all training.
For the biharmonic problem, we use learning rate $10^{-5}$ for the meta-solver and $5.0 \times 10^{-4}$ for the surrogate model.
For the elasticity problem, we use learning rate $10^{-4}$ for the meta-solver and $10^{-4}$ for the surrogate model.
The batch size is set to 256 for all training. 
We train them for 1,000 epochs and reduce the learning rate by a factor of 0.1 at 500 and 750 epochs.
The best model is selected based on the validation loss.

\newpage
\theoremstyle{plain}
\newtheorem*{theorem*}{Theorem}
\newtheorem*{lemma*}{Lemma}
\setcounter{equation}{0}

\end{document}